\documentclass[11pt]{article}
\usepackage[english]{babel}
\usepackage[left=1in,right=1in,top=1in,bottom=1.5in]{geometry} 
\usepackage[skip=5pt, indent=10pt]{parskip}
\usepackage{amssymb,amsmath, amsthm}
\usepackage{hyperref}
\usepackage{xcolor}
\usepackage{nicefrac}
\usepackage[margin=1cm]{caption}
\usepackage{natbib}
\usepackage[makeroom]{cancel}
\usepackage{bbm}
\usepackage{mathrsfs}
\usepackage{times}
\usepackage{graphicx}
\usepackage{subcaption}
\usepackage{multirow}
\usepackage{diagbox}
\usepackage{siunitx}
\usepackage{cleveref}
\usepackage{xargs}
\usepackage{algorithmic}
\usepackage{caption}

\usepackage{url}

\RequirePackage{cleveref}
\usepackage{algorithm}
\usepackage{algorithmic}
\usepackage{caption}
\usepackage{enumitem}

\definecolor{darkred}{HTML}{880000}
\definecolor{darkblue}{HTML}{000088}

\hypersetup{
  colorlinks,
  linkcolor={darkred},
  citecolor={darkblue}
}

\newtheorem{proposition}{Proposition}[section] % numbered within sections
\newtheorem{lemma}[proposition]{Lemma}         % share counter with proposition
\newtheorem{theorem}[proposition]{Theorem}
\theoremstyle{remark}
\newtheorem{remark}[proposition]{Remark}
\newtheorem{corollary}[proposition]{Corollary}

\numberwithin{equation}{section}

\newcommand{\C}{\mathbb{C}}
\newcommand{\R}{\mathbb{R}}
\newcommand{\N}{\mathbb{N}}

\def\RefLaw{\mathsf{Q}}
\newcommand{\E}{\mathbb{E}}

\renewcommand{\Re}{\operatorname{Re}}

\DeclareMathOperator*{\supp}{sup}

\def\amin{a_{-}}
\def\amax{a_{+}}
\def\bmin{b_T^{-}}
\def\bmax{b_T^{+}}
\def\cmin{c_{-}}
\def\cmax{c_{+}}
\def\cTmin{c_T^{-}}
\def\cTmax{c_T^{+}}
\def\rhomin{\rho_{0,-}}
\def\aGmin{a_{\mathcal{G}}}

\def\cGmin{c_{\mathcal{G}}}
\def\cGmax{c_{\mathcal{G}}^{+}}
\def\aStarmin{a_{-}^\star}
\def\aStarmax{a_{+}^\star}
\def\cStarmin{c_{-}^\star}
\def\cStarmax{c_{+}^\star}
\def\supp{\mathrm{supp}}
\def\clip{\mathrm{clip}}

\title{Schr\"odinger bridge problem via empirical risk minimization}

\author{
Denis Belomestny\thanks{Duisburg-Essen University, Germany}
\and
Alexey Naumov\thanks{HSE University, Russian Federation}
\and
Nikita Puchkin\footnotemark[2]
\and
Denis Suchkov \footnotemark[2]
}

\date{}

\begin{document}

\maketitle

\begin{abstract}
We study the Schr\"odinger bridge problem when the endpoint distributions are available only through samples. Classical computational approaches estimate Schr\"odinger potentials via Sinkhorn iterations on empirical measures and then construct a time-inhomogeneous drift by differentiating a kernel-smoothed dual solution. In contrast, we propose a learning-theoretic route: we rewrite the Schr\"odinger system in terms of a single positive \emph{transformed} potential that satisfies a nonlinear fixed-point equation and estimate this potential by \emph{empirical risk minimization} over a function class. We establish uniform concentration of the empirical risk around its population counterpart under sub-Gaussian assumptions on the reference kernel and terminal density. We plug the learned potential into a stochastic control representation of the bridge to generate samples. We illustrate performance of the suggested approach with numerical experiments.
\end{abstract}

\section{Introduction}
The Schr\"odinger bridge problem (SBP) provides a principled way to interpolate between two probability distributions by selecting, among all stochastic processes matching prescribed endpoint marginals, the one that is closest to a reference dynamics in relative entropy. Formally, let $(X_t)_{t\in[0,T]}$ be a Markov process on $\mathbb{R}^d$ with reference law
$\RefLaw$ on path space and transition densities $(q_t)_{t\in(0,T]}$ with respect to the Lebesgue measure, so that
$
\RefLaw\bigl(X_T \in [y, y + dy) \,\big|\, X_0 = x\bigr)
= q_T(x,y)\,dy$. 
We are given two probability densities $\rho_0$ and $\rho_T$ on $\R^d$
and consider the class
\[
\mathcal{P}(\rho_0,\rho_T)
\;:=\;
\Bigl\{\, P \ll \RefLaw \,\Big|\,
        P\circ X_0^{-1} = \rho_0\,dx,\;
        P\circ X_T^{-1} = \rho_T\,dx
\Bigr\}.
\]
The (dynamic) SBP consists in finding the
probability measure $P^\star\in\mathcal{P}(\rho_0,\rho_T)$ which is closest
to $\RefLaw$ in the sense of relative entropy:
\begin{equation}\label{eq:SBP-dynamic}
  P^\star
  \;\in\;
  \arg\min_{P\in\mathcal{P}(\rho_0,\rho_T)}
  \mathcal{H}(P\,\|\,\RefLaw),
  \qquad
  \mathcal{H}(P\,\|\,\RefLaw)
  := \int \log\!\left(\frac{dP}{d\RefLaw}\right)\,dP.
\end{equation}
It is a classical result (see, e.g.,~\cite{ChenGeorgiouPavon2016}) that the minimizer $P^\star$
exists under mild conditions and has a {\em Schr\"odinger factorization} of the form
\begin{equation}\label{eq:schro-factorization}
  \frac{dP^\star}{d\RefLaw}(X)
  \;=\;
  \nu_0(X_0)\,\nu_T(X_T),
\end{equation}
for some nonnegative measurable functions (Schr\"odinger potentials)
$\nu_0,\nu_T : \mathbb{R}^d \to (0,\infty)$.
Taking time--$0$ and time--$T$ marginals in \eqref{eq:schro-factorization} yields
the system
\begin{align}
  \rho_0(x)
  &= \nu_0(x)\int_{\mathbb{R}^d} q_T(x,z)\,\nu_T(z)\,dz,
  \label{eq:schro-system-0}
  \\
  \rho_T(y)
  &= \nu_T(y)\int_{\mathbb{R}^d} q_T(x,y)\,\nu_0(x)\,dx.
  \label{eq:schro-system-T}
\end{align}
The corresponding ``static'' Schr\"odinger bridge problem is the entropy minimization
over couplings of $(X_0,X_T)$:
\begin{equation}\label{eq:SBP-static}
  \pi^\star \;\in\;
  \arg\min_{\pi\in\Pi(\rho_0,\rho_T)}
  \mathcal{H}\bigl(\pi \,\|\, \pi^{\mathrm{ref}}\bigr),
  \qquad
  \pi^{\mathrm{ref}}(dx,dy)
   = \rho_0(x)\,q_T(x,y)\,dx\,dy,
\end{equation}
where $\Pi(\rho_0,\rho_T)$ denotes the set of probability measures on
$\mathbb{R}^d\times\mathbb{R}^d$ with marginals $\rho_0$ and $\rho_T$.
The optimizer $\pi^\star$ can be written as
\begin{equation}
\label{optimal coupling main text}
    \pi^\star(dx,dy)
  \;=\;
  \nu_0(x)\,q_T(x,y)\,\nu_T(y)\,dx\,dy,
\end{equation}
and the potentials $(\nu_0,\nu_T)$ solve \eqref{eq:schro-system-0}--\eqref{eq:schro-system-T}.
In the case $q_T(x)$ is the transition density of the SDE $dX_t = 
\sigma dW_t$ with some $\sigma > 0$,  \eqref{optimal coupling main text} coincides with the solution of the entropy-regularized optimal transport (EOT). From a computational standpoint, EOT is often solved by Sinkhorn iterations on empirical measures, yielding discrete approximations of the Schr\"odinger potentials. Building on this viewpoint, recent works plug these (discrete) potentials into a stochastic control representation of the bridge, producing a drift field after additional smoothing and differentiation. While effective in moderate dimensions, this pipeline raises two conceptual challenges from a learning perspective. First, the potential is estimated only on the support of the empirical samples, and must be extended off-sample (e.g.\ via kernel smoothing) to yield a continuous drift. Second, the overall error blends optimization error (finite Sinkhorn iterations), statistical error (finite samples), and discretization/smoothing error, which complicates generalization analysis.
\par
In this paper, we propose to estimate Schr\"odinger potentials directly in a functional form as a learning problem. We rewrite the Schr\"odinger system in terms of a single positive transformed potential $g$ that satisfies a nonlinear fixed-point equation
\begin{equation}
g = \mathcal{C}[g],
\end{equation}
for a nonlinear integral operator $\mathcal{C}$ depending on $\rho_0,\rho_T$, and $q_T$. When only samples are available,
\[
X_1,\dots,X_N \sim \rho_0,
\qquad
Y_1,\dots,Y_M \sim \rho_T,
\]
we form an empirical operator $\widehat{\mathcal{C}}_{N,M}$ by replacing expectations with empirical averages. Rather than enforcing the fixed point through iterative proportional fitting (Sinkhorn), we estimate $g$ by minimizing an empirical residual loss over a hypothesis class $\mathcal{G}$ of positive functions (e.g.\ neural networks):
\begin{equation}
\widehat{g}_{N,M} \in \arg\min_{g\in\mathcal{G}}
\frac{1}{M}\sum_{j=1}^M
\ell\!\left(g(Y_j),\,\widehat{\mathcal{C}}_{N,M}[g](Y_j)\right),
\end{equation}
where $\ell$ is minimized at equality (e.g.\ squared loss). The resulting objective can be optimized using stochastic-gradient methods. Crucially, in contrast to Sinkhorn-type algorithms, which implicitly produce potentials only at sampled locations and require interpolation or smoothing to obtain continuous objects, the learned estimator is continuous by construction. This is essential in downstream tasks such as:
computing drift fields for Schrödinger bridges, simulating controlled diffusions and
sensitivity analysis and gradient-based control. Moreover, a key benefit of the functional-learning viewpoint is that it naturally accommodates sparse representations of the Schrödinger potential $g$. Depending on the choice of hypothesis class,  one may obtain: sparse basis expansions (e.g. wavelets, Fourier features, or kernel dictionaries), low-rank representations induced by bottleneck neural networks and
implicit sparsity through regularization. Such sparsity can lead to faster evaluation of the potential and its gradients. In translation-invariant settings, sparsity can be particularly effective, as the potential often exhibits low-frequency structure or localized features that can be captured with a small number of active components.
\par
Once a potential (or log-potential) is learned, we use the stochastic control representation of the Schr\"odinger bridge; see \citep{Daipra1991}. Let
\(h(t, x)\) denote the time-evolved Schrödinger potential given by
\[
    h(t, x) = \int \nu_T(y) \, q_{T-t}(x, y) \, dy.
\]
Then the transition density of the process $(X_t^\star)_{t \in (0, T]}$ with the law $P^\star$ is given by
\[
q^h(y, T | x, t) = \frac{q_{T-t}(x, y)h(T, y)}{h(x,t)}.
\]
In the case of diffusion processes  $dX_t = b dt + \sigma d W_t$ this corresponds to the change of drift of basic process $(X_t)_{t \in (0, T]}$ by $a \nabla \log h$ where $a = \sigma \sigma^\top$. This yields a practical sampler: starting from $x_0\sim\rho_0$, we simulate an SDE with the learned drift (e.g.\ via Euler--Maruyama) to obtain approximate bridge samples at intermediate times and at time $T$. 

\paragraph{Contributions} Our key \emph{contributions} could be summarized as follows: 
\begin{itemize}[noitemsep, nolistsep]
\item We reformulate the Schr\"odinger system as a single nonlinear fixed-point equation for a transformed potential $g$, and propose an ERM estimator based on minimizing the empirical fixed-point residual over some class of transformed potentials $\mathcal{G}$.
This gives a flexible framework for the study of the empirical Schr\"odinger problem.

\item When the reference kernel $\RefLaw$ is Gaussian, we show that the population fixed point $g^\star$ admits a rapidly converging  Hermite function expansion.
We derive an explicit $L^2$ approximation bound for the degree-$n$ Hermite function projector and combine it with the
uniform concentration bound for the empirical risk to obtain an end-to-end risk guarantee  with near-parametric dependence on sample size up to polylog factors.

\item 
We numerically illustrate the performance of the method on (i) two-dimensional Swiss roll to S-curve example, (ii) Gaussian mixture transport under train--test shift,  and (iii) single-cell population interpolation. We demonstrate performance that are comparable or improve on existing baselines.
\end{itemize}

\subsection{Related work}
The SBP originates in Schr\"odinger’s 1932 work \cite{Schrodinger1932} on the most likely
stochastic evolution between two prescribed marginals under a reference dynamics.
Modern treatments emphasize its connections to reciprocal processes, large deviations, and optimal
transport; see, e.g., the survey of \citet{Leonard2014} for background and further references.
A stochastic control viewpoint on SBP (and related reciprocal diffusions) appears in
\citet{Daipra1991}, and computational perspectives exploiting projective/Hilbert-metric structure
were developed in, e.g., \citet{ChenGeorgiouPavon2016}.

\paragraph{Entropy-regularized optimal transport and Sinkhorn}
The static SBP coincides with EOT, which has become a
central tool in computational OT due to its stability and algorithmic efficiency.
EOT can be solved in the dual by the Sinkhorn algorithm (iterative proportional fitting / matrix scaling),
popularized in ML by \citet{Cuturi2013} and rooted in earlier matrix-scaling results such as
\citet{FranklinLorenz1989}; see also \citet{PeyreCuturi2019} for a comprehensive overview.
Recent theoretical work has sharpened our understanding of Sinkhorn’s contraction and convergence
properties beyond classical bounded/compact settings and has provided non-asymptotic bounds for
the iterates and their gradients; see, e.g., \citet{ConfortiDurmusGreco2023,GrecoNobleConfortiDurmus2023}.

\paragraph{Estimating Schr\"odinger bridges from samples}
In the statistical setting where $\rho_0$ and $\rho_T$ are only accessible through samples, a standard
approach is to solve EOT between empirical measures (via Sinkhorn) to obtain discrete approximations
of the Schr\"odinger potentials and then construct a sampler for the dynamic bridge.
A representative recent instance is SinkhornBridge of \citet{PooladianNilesWeed2024}, which plugs
(approximate) dual solutions into a stochastic control representation to produce a time-inhomogeneous
drift.
Our approach differs at the estimation stage: instead of computing discrete potentials by Sinkhorn
iterations on empirical measures and subsequently extending/smoothing them, we estimate a
\emph{continuous} potential by ERM over a function class (e.g.\ neural
networks). This viewpoint is tailored to learning-theoretic analysis (uniform concentration, approximation
error) and yields a potential that generalizes off-sample by construction.

\paragraph{SBP in generative modeling and data-to-data translation}
A growing body of work connects SBP to modern generative modeling via controlled diffusions and
score-based methods. Examples include diffusion Schr\"odinger bridges and their applications to
generative modeling \citep{DeBortoliThorntonHengDoucet2021}, learning-based SB variants such as
neural Lagrangian Schr\"odinger bridges \citep{KoshizukaSato2022}, and Schr\"odinger bridge matching
objectives \citep{ShiDeBortoliCampbellDoucet2023}.
Related computationally efficient alternatives include LightSB \citep{KorotinGushchinBurnaev2024}, LightSB-OU 
\citep{puchkin2025tightboundsschrodingerpotential}. 
In applications to biological time interpolation and dynamical modeling, SB/OT ideas also appear in,
e.g., TrajectoryNet \citep{TongHuangWolfVanDijkKrishnaswamy2020} and subsequent simulation-free or
matching-based formulations \citep{TongEtAl2023b}, as well as minibatch OT-based training objectives
\citep{TongEtAl2023a}. Continuous normalizing-flow baselines are often trained with stabilizing
regularization and architectural constraints; see, e.g., \citet{FinlayJacobsenNurbekyanOberman2020}.

\paragraph{Notations}
Given $l, u: \R^d \to \R$ the bracket $[l,u]$ is defined as the collection of all functions $f : \R^d \to \R$ for which $l(x)  \le f(x)  \le u(x)$ for all $x \in \R^d$. For a class $\mathcal F$ of functions $f: \R^d \to \R$ define its bracketing number as follows $\mathcal{N}_{[]}\!\left(\mathcal{F},\|\cdot\|_\infty,\delta\right)
:= \inf\Bigl\{\, m\in\mathbb{N} :
\exists (l_j,u_j)_{j=1}^m\ \text{s.t. } \|u_j-l_j\|_\infty \le \delta,\ 
\mathcal{G}\subseteq \bigcup_{j=1}^m [l_j,u_j] \Bigr\}$. For $f: \R^d \to \R$ define its Fourier transform by $\widehat{f}(\omega):=\int_{\R^d} f(x)e^{-i\omega\cdot x}\,dx$. 
For $K \subset \R^d$ we use notation $\|f\|_{L^\infty(K)}:=\sup_{x \in K} |f(x)|$. For a probability measure $\mu$ on $(\R^d, \mathcal B(\R^d))$ we write $\|f\|_{L^p(\mu)}: = \int_{\R^d} |f(x)| \mu(dx)$. We also define clipping by $\clip(f, a, b) = f$ if $a \le f \le b$, $a$ if $f < a$ and $b$ if $f > b$.

\section{Estimator based on the empirical risk minimization}

First we  recast \eqref{eq:schro-system-0}, \eqref{eq:schro-system-T} in terms of a single transformed potential and obtain a nonlinear
fixed--point problem. Define
\begin{equation}\label{eq:def-g}
  g(y) \;:=\; \frac{\rho_T(y)}{\nu_T(y)},
\end{equation}
that is, $\nu_T(y) = \rho_T(y)/g(y)$. From \eqref{eq:schro-system-0} we obtain
\begin{equation}\label{eq:nu0-in-terms-of-g}
  \nu_0(x)
  \;=\;
  \frac{\rho_0(x)}{\displaystyle \int q_T(x,z)\,\nu_T(z)\,dz}
  \;=\;
  \frac{\rho_0(x)}{\displaystyle \int q_T(x,z)\,\frac{\rho_T(z)}{g(z)}\,dz}.
\end{equation}
For each $x\in\mathbb{R}^d$ define
\begin{equation}
\label{D definition intro}
  D_{g}(x)
  :=
  \int_{\mathbb{R}^d} q_T(x,z)\,\frac{\rho_T(z)}{g(z)}\,dz.
\end{equation}
Substituting \eqref{eq:nu0-in-terms-of-g} and \eqref{eq:def-g} into
\eqref{eq:schro-system-T} gives
\[
\begin{aligned}
  \rho_T(y)
  &= \frac{\rho_T(y)}{g(y)}\int_{\mathbb{R}^d}
     \frac{q_T(x, y) \rho_0(x)}{D_g(x)}\,dx.
\end{aligned}
\]
Canceling the common factor $\rho_T(y)$ on both sides and multiplying by $g(y)$,
we obtain the nonlinear fixed--point equation
\begin{equation}\label{eq:def-C-operator}
  g(y) =
  \int_{\mathbb{R}^d}
    \frac{q_T(x, y) \rho_0(x)}{D_g(x)}
    \,dx =:\mathcal{C}[g](y) \, 
\end{equation}
so the (static) Schr\"odinger bridge problem can be equivalently formulated
as the problem of finding a positive function $g$ solving the nonlinear
fixed--point equation
\(g = \mathcal{C}[g]\). 
Once a fixed point $g$ of $\mathcal{C}$ is found, the Schr\"odinger potentials
are recovered via $\nu_T(y) = \rho_T(y)/g(y)$ and \eqref{eq:nu0-in-terms-of-g},
and the optimal Markov process $P^\star$ is obtained by tilting the reference
process $\RefLaw$ according to \eqref{eq:schro-factorization}.

In many applications the marginal densities $\rho_0$ and $\rho_T$ are not
available explicitly. Instead, we observe independent samples
\[
  X_1,\ldots,X_N \sim \rho_0,
  \qquad
  Y_1,\ldots,Y_M \sim \rho_T,
\]
and seek to recover the fixed point $g$ solving $g = \mathcal{C}[g]$. This gives rise to a
statistical version of the Schr\"odinger bridge problem. Define the empirical measures
\[
  \widehat{\rho}_0^N := \frac{1}{N}\sum_{i=1}^N \delta_{X_i},
  \qquad
  \widehat{\rho}_T^M := \frac{1}{M}\sum_{j=1}^M \delta_{Y_j}.
\]
Replacing $\rho_T$ by its empirical counterparts in \eqref{D definition intro} we obtain 
\begin{equation*}
    \widehat D(x):= \frac{1}{M} \sum_{k=1}^M 
          q_T(x, Y_k)\,\frac{1}{g(Y_k)} \, .
\end{equation*}
In Lemma \ref{lemm: key properties} we show that $\overline{D} \geq D_g(x) \geq \underline{D}$ for all $x \in \supp(\rho_0)$. But we can't prove the same lower bound for $\widehat D(x)$. We introduce a clipping operator $\mathcal{T}_{[a,b]}$  of the form 
$\mathcal{T}_{[a,b]}[f](x) = \clip(f, a, b)$.
Replacing $\rho_0$ and $\rho_T$ by their empirical counterparts in
\eqref{eq:def-C-operator} yields the empirical nonlinear operator
\begin{equation}\label{eq:empirical-C}
  \widehat{\mathcal{C}}_{N,M}[g](y)
  :=
  \frac{1}{N}\sum_{i=1}^N
     \frac{
        q_T(X_i, y)
     }{
       \mathcal{T}_{[\underline{D},\overline{D}]}[\widehat D](X_i)
     }.
\end{equation}
The fixed--point condition $g = \mathcal{C}[g]$ is approximated by enforcing
\[
  g(Y_j) \;\approx\; \widehat{\mathcal{C}}_{N,M}[g](Y_j),
  \qquad j = 1,\dots,M.
\]
To estimate $g$, we recast this system as an ERM
problem. Let $\ell:(0,\infty)\times(0,\infty)\to[0,\infty)$ be a loss function
minimized at equality, e.g. $\ell(u,v)=\tfrac12(u-v)^2$. The empirical risk is defined by
\begin{equation}\label{eq:empirical-risk}
  \widehat{\mathcal{R}}_{N,M}(g)
  :=
  \frac{1}{M}\sum_{j=1}^M
  \ell\!\left(
      g(Y_j),\,
      \widehat{\mathcal{C}}_{N,M}[g](Y_j)
  \right).
\end{equation}
The statistical Schr\"odinger bridge estimator is then given by
\begin{equation}\label{eq:erm-problem}
  \widehat{g}_{N,M}
  \;\in\;
  \arg\min_{g\in\mathcal{G}}
  \widehat{\mathcal{R}}_{N,M}(g),
\end{equation}
where $\mathcal{G}$ is an admissible class of positive functions or
parametric models (e.g.\ neural networks). In the next section, we study the behavior of $\widehat{\mathcal{R}}_{N,M}(g)$ under a set of realistic assumptions that allow us to apply tools from empirical process theory.

\section{Main results}
We state the main assumptions used throughout the paper. We start from the assumptions on the kernel $\RefLaw$ and densities $\rho_0, \rho_T$. 
\begin{itemize}[noitemsep, nolistsep]
  \item[(Q)] 
   There exist constants $\cmin, \cmax>0$ and $\amin, \amax >0$ such that
  \[
    \cmin \exp(-\amin \|x-y\|^2)
    \;\le\;
    q_T(x,y)
    \;\le\;
    \cmax \exp(-\amax \|x-y\|^2),
    \qquad \forall x,y\in \R^d.
  \]
  Moreover, $q_T$ is globally Lipschitz with constant $L_q.$
  \item[(R0)] 
  There exist $x_0\in \R^d$ and two positive real numbers $r_0,R_0$   such that 
  $B(x_0,r_0)\subseteq\mathrm{supp}(\rho_0)\subseteq B(x_0,R_0).$
  Moreover, there is a constant $\rhomin >0$ such that
  \[
    \rho_0(x)\ge \rhomin, \qquad \forall x\in B(x_0,r_0).
  \]
  \item[(RT)] 
  There exist constants $0< \cTmin\le \cTmax<\infty$ and $\bmin, \bmax>0$ such that
  \[
    \cTmin \exp(-\bmin \|y\|^2)
    \;\le\;
    \rho_T(y)
    \;\le\;
    \cTmax \exp(-\bmax \|y\|^2),
    \qquad \forall y\in\mathbb{R}^d.
  \]
  \end{itemize}

The nondegeneracy condition in \emph{(R0)} prevents the normalization factors from becoming
arbitrarily small on $\mathrm{supp}(\rho_0)$ and is used to obtain uniform bounds for
$D_g(x)$ and two-sided bound for the fixed point $g^\star$. Note that we don't assume that $\rho_T$ is compactly supported. Instead, we assume two-sided sub-Gaussian behavior. We now impose additional assumptions on the loss.
\begin{itemize}[noitemsep, nolistsep]
  \item[(L)] The loss $\ell:(0,\infty)\times(0,\infty)\to\mathbb{R}$ is locally bounded
              and jointly Lipschitz, i.e. for any $K>0$ there exist constants $B_\ell=B_\ell(K),L_\ell=L_\ell(K)>0$
              such that for all $(u,v),(u',v')\in(0,K)^2$,
              \[
                |\ell(u,v)| \le B_\ell,
                \qquad
                |\ell(u,v)-\ell(u',v')|
                \le
                L_\ell\bigl(|u-u'| + |v-v'|\bigr).
              \]
\end{itemize}
Finally, we impose assumptions on the hypothesis class.
\begin{itemize}[noitemsep, nolistsep]
 \item[(G)] Assume  that the class
$\mathcal G$ satisfies the following assumptions:
there exist constants $\cGmin > 0$ and $\aGmin >0$ such that for all
$g\in\mathcal G$,
\begin{equation}\label{eq:G-gauss-control}
  \cGmin e^{-\aGmin \|y\|^2} \leq g(y) \leq \cGmax , 
  \qquad \forall y\in\mathbb{R}^d\, .
\end{equation}
\end{itemize}

Note that \emph{(G)} enforces that all candidates $g\in\mathcal G$ are uniformly bounded above and,
critically, bounded below by a Gaussian. The lower bound prevents instabilities caused by the ratio
$q_T(x,z)/g(z)$ in the normalizer $D_g$ and yields a manageable envelope for empirical-process
arguments. Below we additionally assume that  $b_T^+>4a_G$ which ensures that these
envelopes are square-integrable under $\rho_T$.

\begin{theorem}
\label{thm:uniform-concentration-risk}
Suppose that the assumptions \emph{(Q)}, \emph{(R0)}, \emph{(RT)}, \emph{(L)}, \emph{(G)} hold with $\bmax > 4 \aGmin$ and $K = \cGmax(1+(\cmax/\underline{D})$. 
Then we have  for all 
$N,M\ge1$,
\begin{equation}
\label{minimizer of R}
\E [\mathcal{R}(\widehat{g}_{N,M})] \le \inf_{g \in \mathcal G} \mathcal{R}(g) + 2 \E
  \Big[
    \sup_{g\in\mathcal{G}}
    \bigl|
      \widehat{\mathcal{R}}_{N,M}(g)
      -
      \mathcal{R}(g)
    \bigr|
  \Big]\, 
  \end{equation}
and
\begin{equation*}
  \E
  \Big[
    \sup_{g\in\mathcal{G}}
    \bigl|
      \widehat{\mathcal{R}}_{N,M}(g)
      -
      \mathcal{R}(g)
    \bigr|
  \Big] \le  C_1\left(\sqrt{\frac{\log\log(N)}{N}}+\frac{1}{\sqrt{M}}\right) \int_{0}^{C_2} \sqrt{\log \mathcal N_{[]}\!\Big(\mathcal G,\|\cdot\|_\infty,\frac{\varepsilon}{C_3}\Big)} \, d\varepsilon \, ,
\end{equation*}
where $C_1,C_2$ and $C_3$ are positive constants depending on constants from the assumptions \emph{(Q)}, \emph{(R0)}, \emph{(RT)}, \emph{(L)} and \emph{(G)}. 
\end{theorem}

\begin{proof}[Sketch of proof] One can find a detailed proof in the section \ref{sec uniform concentration}. First note that \eqref{minimizer of R} follows from
$$
\mathcal{R}(\widehat{g}_{N,M}) - \inf_{g \in \mathcal G} \mathcal{R}(g) = \mathcal{R}(\widehat{g}_{N,M}) - \widehat{\mathcal{R}}(\widehat{g}_{N,M}) + \inf_{g \in \mathcal G}\widehat{\mathcal{R}}(g)  - \inf_{g \in \mathcal G} \mathcal{R}(g)\, .
$$
To prove the concentration bound we rewrite
\[
  \widehat{\mathcal{R}}_{N,M}(g) - \mathcal{R}(g)
  = \bigl(\widehat{\mathcal{R}}_{N,M}(g) - \widetilde{\mathcal{R}}_M(g)\bigr)+ \bigl(\widetilde{\mathcal{R}}_M(g) - \mathcal{R}(g)\bigr)\, ,
\]
where $\widetilde{\mathcal{R}}_M(g)
  :=
  \frac{1}{M}\sum_{j=1}^M
  \ell\bigl(g(Y_j),\mathcal{C}[g](Y_j)\bigr)$ is an intermediate risk. 
\paragraph{Step 1} Let $\mathcal{F}
  := \{ f_g : g\in\mathcal{G}\}$, where $f_g(y)
  :=
  \ell(g(y) ,\mathcal{C}[g](y))$. Lemma \ref{lemm: key properties} implies that $\mathcal{C}[g]$ is Lipschitz and bounded. More precisely, there exists a constant $L_{\mathcal C, \infty} > 0$ such that
\[
  |\mathcal{C}[g](y) - \mathcal{C}[g'](y)|
  \le
  L_{\mathcal C, \infty} \|g-g'\|_{L^\infty(\mathbb{R}^d)},\quad |\mathcal{C}[g](y)|\leq (\cmax/\underline{D})\|g\|_{L^\infty(\mathbb{R}^d)}.
\]
By  assumption \emph{(L)} and boundedness of any $g\in\mathcal{G}$,
\[
  |f_g(y) - f_{g'}(y)|
  \le
  L_\ell\,\Bigl(
    |g(y) - g'(y)| + |\mathcal{C}[g](y) - \mathcal{C}[g'](y)|
  \Bigr)
\]
with $L_\ell=L_\ell\bigl(K\bigr).$
Combining with \emph{(G)} and \emph{(RT)}, we obtain that $\mathcal{F}$ has 
envelope $F$ with  $\|F\|_{\infty}\leq B_\ell=B_\ell(K)$.
\cite{GineNickl2016}[Theorem~3.5.13] implies
\begin{equation*}
    \E[\sup_{g 
\in \mathcal G} |\widetilde{\mathcal{R}}_M(g) - \mathcal{R}(g)|]\lesssim  \frac{1}{\sqrt{M}}\int_{0}^{8  B_\ell} \sqrt{\log(2 N_{[]}(\mathcal{G}, L^{2}(\rho_T), \varepsilon/L_\ell))} \, d\varepsilon.
\end{equation*}   
\paragraph{Step 2} By assumption \emph{(L)}, 
\[
\E[\sup_{g 
\in \mathcal G} |\widehat{\mathcal{R}}_{N,M}(g) - \widetilde{\mathcal{R}}_M(g)|]
  \le
  L_\ell\,
  \E
  \left[
    \sup_{g\in\mathcal{G}}
      \frac{1}{M}\sum_{j=1}^M
        \bigl|
          \widehat{\mathcal{C}}_{N,M}[g](Y_j) - \mathcal{C}[g](Y_j)
        \bigr|
  \right].
\]
Fix a large compact set $K\subset\mathbb{R}^d$ (to be defined later) and consider 
\[\E
  \bigl[
    \sup_{g\in\mathcal{G}}\sup_{y\in K}
      \bigl|
        \widehat{\mathcal{C}}_{N,M}[g](y) - \mathcal{C}[g](y)
      \bigr|
  \bigr] \le \E[\sup_{g\in\mathcal{G}}\sup_{y\in K}|(\mathrm{A})|]
  +
  \E[\sup_{g\in\mathcal{G}}\sup_{y\in K}|(\mathrm{B})|],
\]
where $(A) = \widehat{\mathcal{C}}_{N,M}[g](y) 
      - \widetilde{\mathcal{C}}_{N}[g](y)$, $(B) = \widetilde{\mathcal{C}}_{N}[g](y) - \mathcal{C}[g](y)$ and 
      $$
  \widetilde{\mathcal{C}}_{N}[g](y)
  :=
  \frac{1}{N}\sum_{i=1}^N
     \frac{q_T(X_i,y)}{D_g(X_i)}.
      $$
      For (B) we define the set $\Phi_K
  := \{
    x\mapsto \phi_{g,y}(x)
   : g\in\mathcal{G}, y\in K\}$, where 
   $$\phi_{g,y}(x)
  := q_T(x,y)/D_g(x).
  $$ 
  Then one may check that
\begin{equation*}
  \log \mathcal N_{[]}\!\big(\Phi_K,L^2(\rho_0), \varepsilon\big)
  \ \le\
  \log \mathcal N_{[]}\!\Big(\mathcal G,\|\cdot\|_\infty, \varepsilon/2L_g\Big)
  \ +\
  C_d\,\log\!\Big(L_y\,\mathrm{diam}(K)/\varepsilon \Big),
\end{equation*}
where $C_d\le d$ is a dimensional constant (coming from covering $K$ in Euclidean norm), $L_g:=c_+\,L_{D,\infty}/\underline D^{2}$ and $L_y:=L_q/\underline D$, with $L_{D,\infty}$ defined in Lemma \ref{lemm: key properties} and $L_q$ defined in \emph{(Q)}.  
Then it follows from  \cite{GineNickl2016}[Theorem~3.5.13], 
\begin{equation*}
  \sup_{g\in\mathcal{G}}\sup_{y\in K}|(\mathrm{B})|
   \lesssim N^{-1/2} \int_{0}^{8 F_{\infty}} \sqrt{\log \mathcal N_{[]}\!\Big(\mathcal G,\|\cdot\|_\infty,\varepsilon/(2L_g)\Big)} \, d\varepsilon +N^{-1/2} \sqrt{\log\!\Big({L_y\,\mathrm{diam}(K)/F_\infty \Big)} }.
\end{equation*}
For (A) we obtain 
\[
  \E
  \left[
    \sup_{g\in\mathcal{G}}\sup_{y\in K}|(\mathrm{A})|
  \right]
  \le
  \cmax L_D\,
  \mathbb{E}
  \left[
    \sup_{g\in\mathcal{G}}
      \frac{1}{N}\sum_{i=1}^N
        \bigl|
          \widehat{D}_g(X_i) - D_g(X_i)
        \bigr|
  \right]\,,
\]
where $L_D := \overline{D}^2/\underline{D}^2$. 
Let $\Psi
  := \{
    (x,z)\mapsto \psi_g(x,z)
   : x\in B(x_0, R_0),\;g\in\mathcal{G}\}$ with 
   $$
   \psi_g(x,z)
  :=
  q_T(x,z)/g(z).
  $$ 
  Note that $\psi_g(x,z)\le F(z):=\frac{c_+}{\cGmin}\,e^{\aGmin\|z\|^2}$ 
  and
  \[
  \|F\|_{L^2(\rho_T)}
  \le
  \frac{c_+}{\cGmin}\,\sqrt{\cTmax}\,
  \Big(\frac{\pi}{\bmax-2\aGmin}\Big)^{d/4}.
\]By \cite{GineNickl2016}[Theorem~3.5.13], 
\begin{align*}
\E\Big[\sup_{x\in B(x_0,R_0)}\sup_{g\in\mathcal G}
  \big|\widehat D_g(x)-D_g(x)\big|\Big]
\lesssim \frac{1}{\sqrt{M}} \int_{0}^{8 \|F\|_{L^2(\rho_T)}} \sqrt{\log(2 N_{[]}(\Psi, L^{2}(\rho_T), \epsilon))} \, d\varepsilon.
\end{align*} 
It remains to estimate the bracketing number in the r.h.s. of the previous inequality. Note that one may show that,
\[
  \log \mathcal N_{[]}\!\big(\Psi,L^2(\rho_T), \varepsilon\big)
  \ \le\
  \log \mathcal N_{[]}\!\Big(\mathcal G,\|\cdot\|_\infty, \varepsilon/C_g\Big)
  \;+\;
  d\,\log\!\Big(C\,R_0/\varepsilon\Big),
\]
where $C>0$ is some absolute constant, \( C_g:=\cmax/(\cGmin^2)\,M_\rho)\) and
\[
M_\rho=\Big(\int e^{4\aGmin\|z\|^2}\rho_T(z)\,dz\Big)^{1/2}.
\]
\paragraph{Step 3} Finally, we take $K = B(0, R)$ with $R > R_0$. 
Using $D_g(x)\ge \underline D$ and that $\rho_0$ is a probability density,
\begin{align*}
  |\mathcal C[g](y)|
  \le
  \frac{1}{\underline D}\int_{B(0,R_0)}\rho_0(x)\,q_T(x,y)\,dx
  \le
  \frac{\cmax}{\underline D}\exp\!\bigl(-\amax (R-R_0)^2\bigr).
\end{align*}
Furthermore, using $\mathcal{T}_{[\underline{D},\overline{D}]}[\widehat D_g](X_i)\ge \underline D$ and $X_i\in B(0,R_0)$,
\begin{align*}
  |\widehat{\mathcal C}_{N,M}[g](y)|
  \le
  \frac{1}{\underline D}\sup_{x\in B(0,R_0)} q_T(x,y)
  \le
  \frac{\cmax}{\underline D}\exp\!\bigl(-\amax (R-R_0)^2\bigr).
\end{align*}
Hence, for every $R>R_0$,
\begin{align*}
\mathbb{E}\Bigg[
  \sup_{g\in\mathcal{G}}\sup_{y\not\in B(0,R)}
      \bigl|
        \widehat{\mathcal{C}}_{N,M}[g](y) - \mathcal{C}[g](y)
      \bigr|
\Bigg]
&\le
\frac{2\cmax}{\underline D}\,
\exp\!\bigl(-\amax (R-R_0)^2\bigr).
\label{eq:tail-outside-ball}
\end{align*}
It remains to choose $R$ of order $\sqrt{ \log N}$. 
\end{proof}
\subsection{Approximation error}
In this section, we show that by choosing an appropriate function class $\mathcal{G}$, we can achieve a small value of $\E [\mathcal{R}(\widehat{g}_{N,M})]$. For simplicity, we assume that the transition kernel is Gaussian.  
Recall that the function has the following form
$$
g^\star(y) = \int_{\R^d} w(x)\,q(y-x)\,dx
$$ 
with $w(x) = \rho_0(x)/D_g(x)$, that is, it represents a convolution of compactly supported function with the Gaussian kernel. In this case, the class of Hermite polynomials is a natural candidate for the class of functions $G$. We have brought to the appendix the main results concerning Hermite polynomials, in particular estimates of Hermite coefficients of $g^\star$.
More precisely, assume the following condition.
\begin{itemize}[noitemsep, nolistsep]
 \item[(Q$_\mathrm{Gauss}$)] The reference kernel is Gaussian:
  \[
    q_T(z)
    =
    (2\pi T)^{-d/2}\exp\!\left(-\frac{\|z\|^2}{2T}\right),
    \qquad z \in \R^d.
  \]
 \end{itemize}
The following theorem provides a bound for $\E [\mathcal{R}(\widehat{g}_{N,M})]$.
\begin{theorem}
    \label{approx theorem risk}
    Assume conditions \emph{(Q$_\mathrm{Gauss}$)}, \emph{(R0)}, \emph{(RT)} with $\bmax > 4/T$ and \emph{(L)}. Then there exists a class $\mathcal G$ such that \emph{(G)} is satisfied  and
    $$
     \E [\mathcal{R}(\widehat{g}_{N,M})] \lesssim 
  \log^{d/2} \max(M, N) \Big(N^{-1/2} + M^{-1/2} \Big)\, , 
    $$
   where $\lesssim$ stands for inequality up to constants independent of  $M$ and $d$ and double logarithmic factors.
\end{theorem}
\begin{proof}
Using $\mathcal R(g^\star) = 0$ and Corollary \ref{cor:excess-risk-L2} we get
\begin{equation}
\label{approx 2step}
    \inf_{g \in \mathcal G} \mathcal{R}(g) \le L_\ell\bigl(1+L_{\mathcal C,2}\bigr) \inf_{g \in \mathcal G}  \|g-g^\star\|_{L^2(\rho_T)} \, .
\end{equation}
Let $\{\psi_\alpha^{(\lambda)}\}_{\alpha\in\N_0^d}$ be the scaled Hermite basis with $\lambda=T^{-1/2}$.  
Let $\Pi_{n}$ be the $L^2(\R^d)$--orthogonal projector onto
$\mathrm{span}\{\psi_\alpha^{(\lambda)}:\ |\alpha|\le n\}$. 
Define 
$$
\mathcal G = \mathcal G_n(B): = \left\{g = \sum_{|\alpha| \le n} c_\alpha \psi_\alpha^{(\lambda)}: \sum_{|\alpha| \le n} c_\alpha^2 \le B\right\}.
$$ 
We choose $B \asymp n d^d$. By Proposition \ref{prop:proj-error-T1-unscaled}, there exist some constant $C > 0$ such that for $n \geq C \log \max(M, N)$, $\Pi_{n}g^\star \in \mathcal G$, and
\begin{equation}\label{eq:proj-error-T main text}
\|g^\star-\Pi_{n}g^\star\|_{L^2(\R^d)}
 \lesssim 
\log^{d/2} \max(M, N) \Big(N^{-1/2} + M^{-1/2} \Big)\, .
\end{equation}
Note that by Proposition \ref{prop:g-two-sided-subgaussian-direct}, $g^\star(y)\in\big[\cStarmin e^{-\aStarmin\|y\|^2},\,\cStarmax \big]: = I$ for all $y$ by
\eqref{eq:g-two-sided-Gaussian}. We need to ensure condition \eqref{eq:G-gauss-control}. Define the clipping operator $\Pi^\clip:\R\to\R$ by 
$$
\Pi^\clip(t):=
  \min\Big\{\cStarmax,\max\big\{t,\cStarmin e^{-\aStarmin\|y\|^2} \big\}\Big\}\,. 
$$ 
Then for any $f \in \mathcal G$, $\Pi^\clip(f)(y)$ satisfies \eqref{eq:G-gauss-control}. For each fixed $y$, the map
$t\mapsto \Pi(t)$ is the metric projection onto the interval
$I$, hence it is $1$--Lipschitz. 
% \[
%   |\Pi(t)-\Pi(s)|\le |t-s|,\qquad \forall s,t\in\R.
% \]
Applying this  gives the pointwise contraction for any $y\in\R^d$, 
\begin{equation*}
  |g^\star(y)-\Pi^\clip(\Pi_{n}g^\star(y))|
  =
  \big|\Pi^\clip (g^\star(y)) - \Pi^\clip(\Pi_{n}g^\star(y))\big|
  \le
  |g^\star(y)-\Pi_{n}g^\star(y)|,
  \qquad 
\end{equation*}
In particular, clipping can only \emph{decrease} any $L^2(
\R^d)$--error: 
\[
  \|g^\star(y)-\Pi^\clip(\Pi_{n}g^\star(y))\|_{L^2(\R^d)}
  \le
  \|g^\star(y)-\Pi_{n}g^\star(y)\|_{L^2(\R^d)}\, .
\]
It follows from the last inequality, \eqref{approx 2step} and \eqref{eq:proj-error-T main text} that 
$$
\inf_{g \in \mathcal G} \mathcal{R}(g) \lesssim \log^{d/2} \max(M, N) \Big(N^{-1/2} + M^{-1/2} \Big)\, . 
$$
Note that for chosen $n$, Lemma \ref{bracketing number} and Proposition \ref{prop:hermite-coeff} imply that there exists some constant $C^\prime >0$ such that
$$
\log \mathcal N_{[]}\!\Big(\mathcal G,\|\cdot\|_\infty, \varepsilon \Big) \lesssim \log^{d/2} \max(M, N) \log(C^\prime/\varepsilon)\, .
$$ 
Hence, we may conclude the statement by applying Theorem \ref{thm:uniform-concentration-risk}.
\end{proof}
 
% \subsection{Distance between $g_{m,R_m}$ and $g^\star$}
% For simplicity we again assume \emph{(Q$_\mathrm{Gauss}$)}. We slightly change normalisation assumptions. First, we assume that 
% $$
% \int_{B_R}\frac{\rho_T(y)}{g^\star(y)}\,dy=1,
% $$
% where $B_R = B(0, R)$ for some $R > 0$. 
% We introduce a class of functions
% \begin{enumerate}[noitemsep, nolistsep]
% \item  For every $g\in\mathcal G,$
% \[
%   \int_{B_R}\frac{\rho_T(y)}{g(y)}\,dy=1. 
% \]
% \item There exists $\gamma\in(0,1]$ such that
% for every $u\in\mathcal G$,
% \begin{equation}\label{eq:global-dom main text}
%   u(y)\ge \gamma\,g^\star(y),\qquad \forall y\in\R^d.
% \end{equation}
% \end{enumerate}
% Note that Suppose that  every $u\in\mathcal G$ satisfies a \emph{uniform} Gaussian lower bound:
% \begin{equation}\label{eq:G-common-exp-lower main text}
%   u(y)\ \ge\ C_{\mathcal G,-}\,e^{-a_2\|y\|^2},\qquad \forall y\in\R^d,
% \end{equation}
% for some $C_{\mathcal G,-}>0$ where \(a_2 := \frac{a_+}{2}\) and \(a_+=1/(2T)\). Then \eqref{eq:global-dom main text} holds with
% \(\gamma=(C_{\mathcal G,-}/C_2)^\star\wedge 1\)
% and \(C_2\) coming from Assumption \emph{(Q)}.

\section{Numerics} \label{sec:numerical}

In this section, we present an experimental analysis of the proposed algorithm and its comparison with SinkhornBridge \citep{PooladianNilesWeed2024}. This algorithm was chosen for comparison because it is closest to ours in terms of its training procedure. However, because our algorithm learns a continuous Schrödinger log-potential function rather than a discrete solution to the optimal transport problem, we were able to demonstrate significant superiority across several generative modelling and data-to-data translation tasks.  
%It should be noted that the formal algorithm for numerical experiments differs from the proposed general ERM framework in several respects. First, the algorithm learns the log Schrödinger potential rather than the function $ g(y) $, since we do not have access to the density of the target distribution in practice, and optimization in the log domain is significantly more stable. Second, when calculating the loss for the fixed-point equation, the loss is pre-centered, which improves the convergence and behavior of the algorithm and also satisfies the normalization condition imposed on the function $ g(y) $ in \eqref{eq:G-gauss-control}.
In the following, the algorithm we proposed will be referred to as ERM-Bridge for brevity. The formal description of the algorithm, the hyperparameter values and additional experimental details are provided in the appendix \ref{sec:hyperparams}.

\begin{figure*}[h!]
    \centering
    \includegraphics[width=0.95\linewidth]{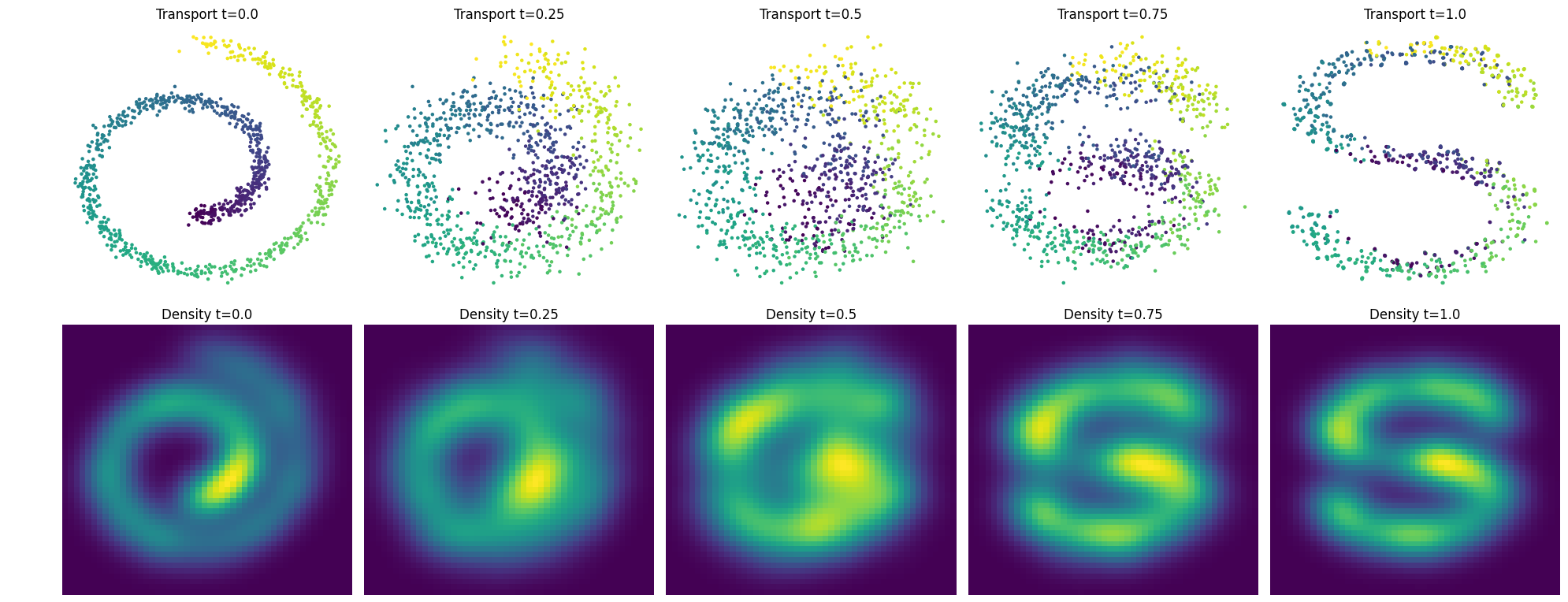}
    \caption{Sample translation from Swiss-Roll to S-Curve and density map for ERM-Bridge at time $ t \in [0, 0.25, 0.5, 0.75, 1]. $}
    \label{fig:SwissRoll}
\end{figure*}

\paragraph{Swiss-Roll to S-Curve Experiment}

We evaluate the validity of our algorithm on classic two dimensional example. In this experiment, we examined the translation of the Swiss-Roll distribution into an S-Curve and present the distribution at time $ t \in [0, 0.25, 0.5, 0.75, 1] $ for clarity, see Figure \ref{fig:SwissRoll}. The figure clearly shows that the algorithm reliably learns the distribution and the correct density map (the density was approximated using KDE on a two-dimensional surface). It is also worth noting that our algorithm demonstrated the best training and sampling time on this problem. The results and hyperparameters are reported in the appendix \ref{sec:hyperparams}.

\paragraph{Evaluation one the Gaussian Mixture}

To further illustrate the importance of the network approximated log-potential, we evaluated our algorithm on a Gaussian mixture problem. For this experiment, we used mixtures of 25 Gaussians from a uniform grid with standard identical covariance matrices. Transport quality was quantified with the sliced approximation of the Wasserstein distance $\mathbb{W}_1$ as the metric, see Appendix \ref{sec:metrics} for definition of metrics.

In this experiment, both algorithms were trained on 3000 samples from a truncated normal distribution on $[-10, 10], $ and sampled on starting points from a truncated normal distribution on $[-1, 1], [-2, 2], [-5, 5]. $ This experiment demonstrates how the algorithms considered behave when the starting distributions in the training set and the sampling set may differ, see Figure \ref{fig:gauss}. This situation is quite typical for data-to-data transport, where both the starting and target distributions are available only from samples, and it is impossible to ensure complete identity between the training and test sets. It is clearly seen that the continuous potential-like function learned by the neural network in the our algorithm behaves better than the discrete optimal transport SinkhornBridge. The baseline achieved $\mathbb{W}_1 = 1.3115  $ for $ [-1, 1] $ case, while our algorithm achieved $\mathbb{W}_1 = 0.3818. $ The hyperparameter values are reported in the appendix \ref{sec:hyperparams}.

\begin{minipage}[t]{0.5\textwidth}
\centering
\captionof{figure}{Plot of sliced Wasserstein distance as a function of KL between the distribution on the train and sampling for ERM-Bridge and SinkhornBridge.}
\label{fig:gauss}
\resizebox{\textwidth}{!}{
\includegraphics[width=0.9\linewidth]{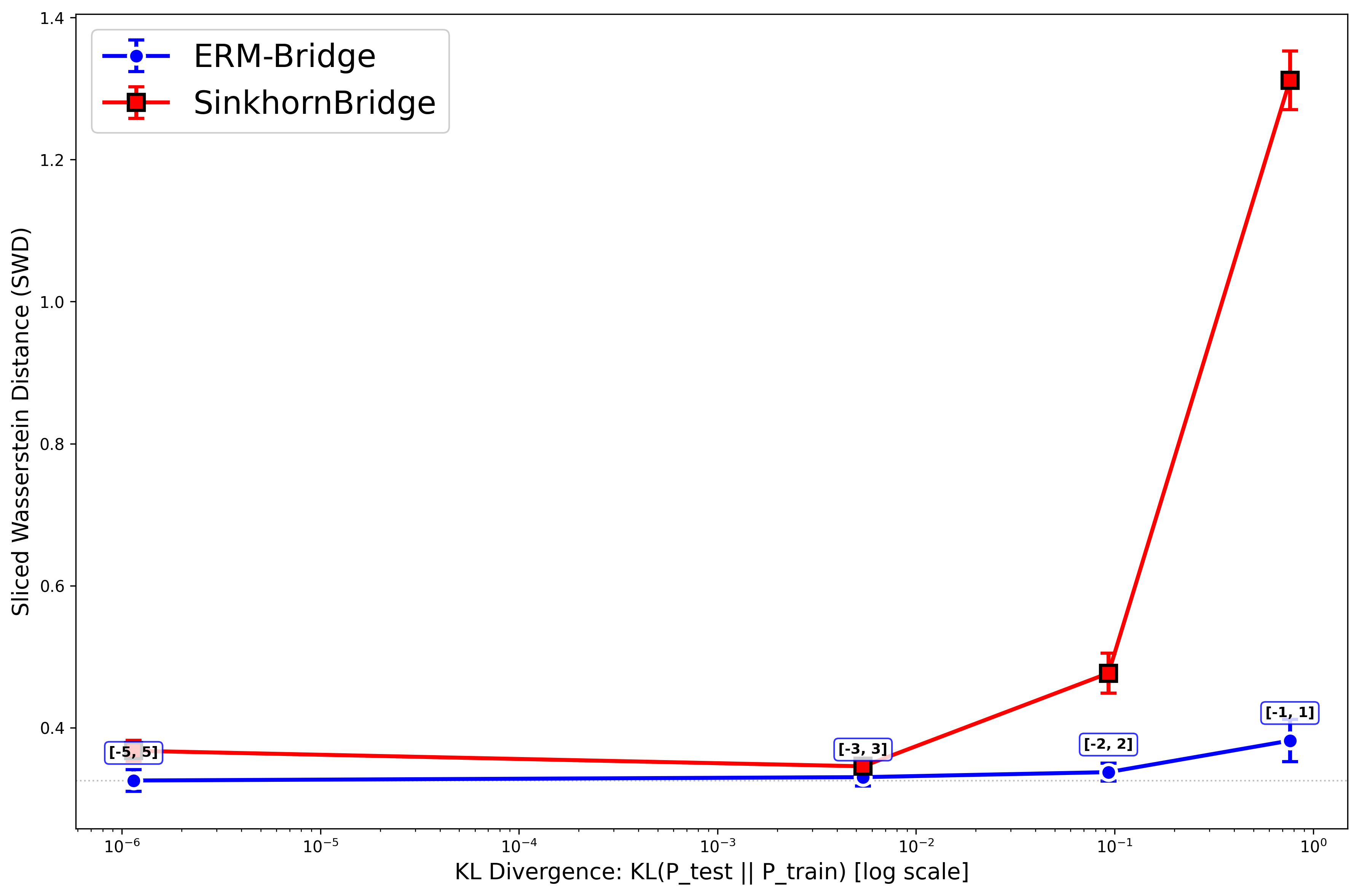}
}
\end{minipage}
\hfill
\begin{minipage}[t]{0.43\textwidth}
\centering
\captionof{table}{The quality of intermediate distribution restoration for single-cell data for various algorithms, including ERM-Bridge and SinkhornBridge.}
\label{tab:cell}
\resizebox{\textwidth}{!}{
\begin{tabular}{|c|c|}
\hline
\textbf{Solver} & $\mathbb{W}_1$ metric \\
\hline
OT-CFM \citep{TongEtAl2023a} & $0.790 \pm 0.068$ \\
\hline
$[\mathrm{SF}]^2$M-Exact \citep{TongEtAl2023b} & $0.793 \pm 0.066$ \\
\hline
\textbf{ERM-Bridge} & $0.809 \pm 0.030$ \\
\hline 
LightSB-OU \citep{puchkin2025tightboundsschrodingerpotential} & $0.815 \pm 0.016$\\
\hline
SinkhornBridge \citep{PooladianNilesWeed2024} & $0.818 \pm 0.028$ \\
\hline
LightSB \citep{KorotinGushchinBurnaev2024} & $0.823 \pm 0.017$ \\
\hline
Reg. CNF \citep{FinlayJacobsenNurbekyanOberman2020} & $0.825 \pm \text{N/A} $\footnote{\label{note:na}The authors did not report the standard deviation.} \\
\hline
T. Net \citep{TongHuangWolfVanDijkKrishnaswamy2020} & $0.848 \pm \text{N/A}$\footref{note:na} \\
\hline
DSB \citep{DeBortoliThorntonHengDoucet2021} & $0.862 \pm 0.023$ \\
\hline
I-CFM \citep{TongEtAl2023a} & $0.872 \pm 0.087$ \\
\hline
$[\mathrm{SF}]^2$M-Geo \citep{TongEtAl2023b} & $0.879 \pm 0.148$ \\
\hline
NLSB \citep{KoshizukaSato2022} & $0.970 \pm \text{N/A}$\footref{note:na} \\
\hline
$[\mathrm{SF}]^2$M-Sink \citep{TongEtAl2023b} & $1.198 \pm 0.342$ \\
\hline
SB-CFM \citep{TongEtAl2023a} & $1.221 \pm 0.380$ \\
\hline
DSBM \citep{ShiDeBortoliCampbellDoucet2023} & $1.775 \pm 0.429$ \\
\hline
\end{tabular}
}
\end{minipage}

\paragraph{Evaluation one the Single Cell Data}

For a more comprehensive comparison with SinkhornBridge, we conducted experiments on biological data \citep{TongHuangWolfVanDijkKrishnaswamy2020}. Following original paper, we formulated the problem as transporting the cell distribution at time $t_{i-1}$ to time $t_{i+1}$ for $i \in \{1, 2, 3\}$. We use results for other methods from \citep{TongEtAl2023b}, whose authors were the first to consider this setup in \citep{TongHuangWolfVanDijkKrishnaswamy2020}. We then predicted the cell distribution at the intermediate time $t_i$ and computed the Wasserstein distance $\mathbb{W}_1$ between the predicted distribution and the ground truth. The results were averaged across all three setups ($i = 1, 2, 3$), see Table \ref{tab:cell}.

To ensure statistical robustness, we repeated the experiment five times. The baseline achieved $\mathbb{W}_1 = 0.818 $, while our algorithm achieved $\mathbb{W}_1 = 0.809 $. The hyperparameter values are reported in the appendix \ref{sec:hyperparams}.

\section{Conclusion}
We studied the \emph{statistical} problem of estimating Schr\"odinger bridge potentials from finite samples. We rewrote the Schr\"odinger system as a single nonlinear fixed-point equation
$g=C[g]$ for a transformed potential, and proposed an ERM estimator obtained by
minimizing an empirical fixed-point residual over a hypothesis class.
This yields a continuous potential estimator by construction, and naturally pairs with the stochastic-control
representation of the Schr\"odinger bridge. Several directions remain for future work, including extensions beyond sub-Gaussian tail
assumptions and bounds on the error between $g_{M,N}$ and $g^\star$. The latter will require to study local behavior of derivative of $\mathcal C[g]$ near $g^\star$ in the Hilbert metric (spectral gap condition).

\bibliographystyle{abbrvnat}
\bibliography{bibliography,approx-biblio}

@article{Schrodinger1932,
  author = {Schr\"odinger, Erwin},
  title = {{\"U}ber die {U}mkehrung der {N}aturgesetze},
  journal = {Sitzungsberichte der Preussischen Akademie der Wissenschaften, Physikalisch-Mathematische Klasse},
  year = {1932},
  pages = {144-153}
}

@article{Leonard2014,
  author = {Leonard, Christian},
  title = {A survey of the {S}chr\"{o}dinger problem and some of its connections with optimal transport},
  journal = {Discrete and Continuous Dynamical Systems},
  year = {2014},
  volume = {34},
  number = {4},
  pages = {1533-1574}
}

@article{Daipra1991,
  author = {Dai Pra, P.},
  title = {A stochastic control approach to reciprocal diffusion processes},
  journal = {Applied Mathematics and Optimization},
  year = {1991},
  volume = {23},
  number = {1},
  pages = {313-329},
  doi = {10.1007/BF01445134}
}

@inproceedings{Akiba2019,
  title={{O}ptuna: A Next-Generation Hyperparameter Optimization Framework},
  author={Akiba, Takuya and Sano, Shotaro and Yanase, Toshihiko and Ohta, Takeru and Koyama, Masanori},
  booktitle={The 25th ACM SIGKDD International Conference on Knowledge Discovery \& Data Mining},
  pages={2623--2631},
  year={2019}
}

@article{ChenGeorgiouPavon2016,
  author  = {Chen, Yongxin and Georgiou, Tryphon T. and Pavon, Michele},
  title   = {Entropic and Displacement Interpolation: A Computational Approach Using the {Hilbert} Metric},
  journal = {{SIAM} Journal on Applied Mathematics},
  volume  = {76},
  number  = {6},
  pages   = {2375--2396},
  year    = {2016},
  doi     = {10.1137/16M1061382},
  url     = {https://doi.org/10.1137/16M1061382}
}

@inproceedings{Cuturi2013,
  author    = {Cuturi, Marco},
  title     = {Sinkhorn Distances: Lightspeed Computation of Optimal Transport},
  booktitle = {Advances in Neural Information Processing Systems},
  volume    = {26},
  pages     = {2292--2300},
  year      = {2013},
  url       = {https://papers.nips.cc/paper/4927-sinkhorn-distances-lightspeed-computation-of-optimal-transport}
}

@article{FranklinLorenz1989,
  author  = {Franklin, Joel and Lorenz, Jens},
  title   = {On the Scaling of Multidimensional Matrices},
  journal = {Linear Algebra and its Applications},
  volume  = {114--115},
  pages   = {717--735},
  year    = {1989},
  doi     = {10.1016/0024-3795(89)90490-4},
  url     = {https://doi.org/10.1016/0024-3795(89)90490-4}
}

@article{PeyreCuturi2019,
  author  = {Peyr{\'e}, Gabriel and Cuturi, Marco},
  title   = {Computational Optimal Transport},
  journal = {Foundations and Trends in Machine Learning},
  volume  = {11},
  number  = {5--6},
  pages   = {355--607},
  year    = {2019},
  month   = feb,
  doi     = {10.1561/2200000073},
  url     = {https://doi.org/10.1561/2200000073}
}

@article{ConfortiDurmusGreco2023,
  author       = {Conforti, Giovanni and Durmus, Alain and Greco, Giacomo},
  title        = {Quantitative Contraction Rates for Sinkhorn Algorithm: Beyond Bounded Costs and Compact Marginals},
  journal      = {arXiv e-prints},
  year         = {2023},
  month        = apr,
  eprint       = {2304.04451},
  archivePrefix= {arXiv},
  doi          = {10.48550/arXiv.2304.04451},
  url          = {https://doi.org/10.48550/arXiv.2304.04451}
}

@inproceedings{GrecoNobleConfortiDurmus2023,
  author    = {Greco, Giacomo and Noble, Maxence and Conforti, Giovanni and Durmus, Alain},
  title     = {Non-asymptotic Convergence Bounds for Sinkhorn Iterates and Their Gradients: A Coupling Approach},
  booktitle = {Proceedings of the 36th Conference on Learning Theory},
  series    = {Proceedings of Machine Learning Research},
  volume    = {195},
  pages     = {716--746},
  year      = {2023},
  publisher = {PMLR},
  url       = {https://proceedings.mlr.press/v195/greco23a.html}
}

@article{PooladianNilesWeed2024,
  author       = {Pooladian, Aram-Alexandre and Niles-Weed, Jonathan},
  title        = {Plug-in Estimation of Schr{\"o}dinger Bridges},
  journal      = {arXiv preprint arXiv:2408.11686},
  year         = {2024},
  doi          = {10.48550/arXiv.2408.11686},
  url          = {https://doi.org/10.48550/arXiv.2408.11686},
  eprint       = {2408.11686},
  archivePrefix= {arXiv}
}

@inproceedings{DeBortoliThorntonHengDoucet2021,
  author    = {De Bortoli, Valentin and Thornton, James and Heng, Jeremy and Doucet, Arnaud},
  title     = {Diffusion Schr{\"o}dinger Bridge with Applications to Score-Based Generative Modeling},
  booktitle = {Advances in Neural Information Processing Systems},
  volume    = {34},
  pages     = {17695--17709},
  year      = {2021}
}

@article{KoshizukaSato2022,
  author       = {Koshizuka, Takeshi and Sato, Issei},
  title        = {Neural Lagrangian Schr{\"o}dinger Bridge: Diffusion Modeling for Population Dynamics},
  journal      = {arXiv preprint arXiv:2204.04853},
  year         = {2022},
  doi          = {10.48550/arXiv.2204.04853},
  url          = {https://doi.org/10.48550/arXiv.2204.04853},
  eprint       = {2204.04853},
  archivePrefix= {arXiv}
}

@article{ShiDeBortoliCampbellDoucet2023,
  author       = {Shi, Yuyang and De Bortoli, Valentin and Campbell, Andrew and Doucet, Arnaud},
  title        = {Diffusion Schr{\"o}dinger Bridge Matching},
  journal      = {arXiv preprint arXiv:2303.16852},
  year         = {2023},
  doi          = {10.48550/arXiv.2303.16852},
  url          = {https://doi.org/10.48550/arXiv.2303.16852},
  eprint       = {2303.16852},
  archivePrefix= {arXiv}
}

@inproceedings{KorotinGushchinBurnaev2024,
  author    = {Korotin, Alexander and Gushchin, Nikita and Burnaev, Evgeny},
  title     = {Light Schr{\"o}dinger Bridge},
  booktitle = {International Conference on Learning Representations},
  year      = {2024},
  doi       = {10.48550/arXiv.2310.01174},
  url       = {https://doi.org/10.48550/arXiv.2310.01174}
}

@inproceedings{TongHuangWolfVanDijkKrishnaswamy2020,
  author    = {Tong, Alexander and Huang, Jessie and Wolf, Guy and van Dijk, David and Krishnaswamy, Smita},
  title     = {TrajectoryNet: A Dynamic Optimal Transport Network for Modeling Cellular Dynamics},
  booktitle = {Proceedings of the 37th International Conference on Machine Learning},
  series    = {Proceedings of Machine Learning Research},
  volume    = {119},
  pages     = {9526--9536},
  year      = {2020},
  publisher = {PMLR},
  url       = {https://proceedings.mlr.press/v119/tong20a.html}
}

@article{TongEtAl2023a,
  author       = {Tong, Alexander and Fatras, Kilian and Malkin, Nikolay and Huguet, Guillaume and Zhang, Yanlei and Rector-Brooks, Jarrid and Wolf, Guy and Bengio, Yoshua},
  title        = {Improving and Generalizing Flow-based Generative Models with Minibatch Optimal Transport},
  journal      = {arXiv preprint arXiv:2302.00482},
  year         = {2023},
  note         = {2023a},
  doi          = {10.48550/arXiv.2302.00482},
  url          = {https://doi.org/10.48550/arXiv.2302.00482},
  eprint       = {2302.00482},
  archivePrefix= {arXiv}
}

@article{TongEtAl2023b,
  author       = {Tong, Alexander and Malkin, Nikolay and Fatras, Kilian and Atanackovic, Lazar and Zhang, Yanlei and Huguet, Guillaume and Wolf, Guy and Bengio, Yoshua},
  title        = {Simulation-Free Schr{\"o}dinger Bridges via Score and Flow Matching},
  journal      = {arXiv preprint arXiv:2307.03672},
  year         = {2023},
  note         = {2023b},
  doi          = {10.48550/arXiv.2307.03672},
  url          = {https://doi.org/10.48550/arXiv.2307.03672},
  eprint       = {2307.03672},
  archivePrefix= {arXiv}
}

@inproceedings{FinlayJacobsenNurbekyanOberman2020,
  author    = {Finlay, Chris and Jacobsen, Joern-Henrik and Nurbekyan, Levon and Oberman, Adam},
  title     = {How to Train Your Neural {ODE}: the World of {J}acobian and Kinetic Regularization},
  booktitle = {Proceedings of the 37th International Conference on Machine Learning},
  series    = {Proceedings of Machine Learning Research},
  volume    = {119},
  pages     = {3154--3164},
  year      = {2020},
  publisher = {PMLR},
  url       = {https://proceedings.mlr.press/v119/finlay20a.html}
}

@book{GineNickl2016,
  author    = {Giné, Evarist and Nickl, Richard},
  title     = {Mathematical Foundations of Infinite-Dimensional Statistical Models},
  year      = {2016},
  publisher = {Cambridge University Press},
  series    = {Cambridge Series in Statistical and Probabilistic Mathematics},
  volume    = {40},
  doi       = {10.1017/CBO9781107337312},
  isbn      = {9781107043169}
}

@misc{puchkin2025tightboundsschrodingerpotential,
      title={Tight Bounds for Schr\"odinger Potential Estimation in Unpaired Data Translation}, 
      author={Nikita Puchkin and Denis Suchkov and Alexey Naumov and Denis Belomestny},
      year={2025},
      eprint={2508.07392},
      archivePrefix={arXiv},
      primaryClass={cs.LG},
      url={https://arxiv.org/abs/2508.07392}, 
}

\appendix

\section{Properties of the empirical operator under sub-Gaussian assumptions}

\begin{proposition}[Two--sided Gaussian bounds for the fixed point]
\label{prop:g-two-sided-subgaussian-direct}
Assume \emph{(Q)}, \emph{(R0)}, \emph{(RT)}. Let $g^\star:\mathbb{R}^d\to(0,\infty)$ be measurable and satisfy
\begin{equation}\label{eq:fixed-point-normalization}
  g^\star(y) = \mathcal{C}[g^\star](y), \qquad y\in\mathbb{R}^d,
  \qquad\text{and}\qquad
  \int_{\mathbb{R}^d} \frac{1}{g^\star(z)}\,\rho_T(z)\,dz = 1.
\end{equation}
Set \(\aStarmin := 2a_-,\) \(\aStarmax := \frac{a_+}{2}.\)
Then there exist constants $\cStarmin, \cStarmax>0$ depending only on
$q_T,\rho_0,\rho_T$ (but not on $g^\star$), such that
\begin{equation}\label{eq:g-two-sided-Gaussian}
  \cStarmin \exp(-\aStarmin \|y\|^2)
  \;\le\;
  g^\star(y)
  \;\le\;
  \cStarmax \exp(-\aStarmax \|y\|^2),
  \qquad \forall y\in\mathbb{R}^d.
\end{equation}
\end{proposition}

\begin{remark}
    It is important to note that the behavior of the function $g^\star$ is determined only by the conditions on the transition kernel $q$ and densities $\rho_0, \rho_T$.
\end{remark}

\begin{proof}
For each $x\in\mathbb{R}^d$ define
\[
  D_{g^\star}(x)
  :=
  \int_{\mathbb{R}^d} q_T(x,z)\,\frac{\rho_T(z)}{g^\star(z)}\,dz.
\]
Then the fixed--point equation can be written as
\[
  g^\star(y)
  =
  \int_{\mathbb{R}^d}
    \frac{\rho_0(x)}{D_{g^\star}(x)}\,q_T(x,y)\,dx.
\]
We claim that there exist constants $0<\underline{D}^\star \le\overline{D}^\star<\infty$,
depending only on $q_T,\rho_0,\rho_T$, such that
\begin{equation}\label{eq:Dg-two-sided}
  \underline{D}^\star \;\le\; D_{g^\star}(x) \;\le\; \overline{D}^\star,
  \qquad \forall x\in\mathrm{supp}(\rho_0).
\end{equation}
By \emph{(Q)},
\[
  q_T(x,z) \le \cmax \exp(-\amax \|x-z\|^2) \le \cmax,
\]
hence, using the normalization in \eqref{eq:fixed-point-normalization},
\begin{equation}
\label{D upper boound}
  D_{g^\star}(x)
  =
  \int q_T(x,z)\,\frac{\rho_T(z)}{g^\star(z)}\,dz
  \le
  \cmax \int \frac{\rho_T(z)}{g^\star(z)}\,dz
  =
  \cmax.
\end{equation}
Thus we may take
\begin{equation}
    \label{Dmax def}
  \overline{D}^\star := \cmax.
\end{equation}
Define
\[
  m(x)
  :=
  \int_{\mathbb{R}^d} q_T(x,z)\,\rho_T(z)\,dz.
\]
Using \emph{(Q)}, \emph{(RT)} and the compactness of $B(x_0,R_0)$,
one checks that 
\[
  m_- \;\le\; m(x) \;\le\; m_+, \qquad \forall x\in\mathrm{supp}(\rho_0)
\]
where
\(
  m_+
  :=
  \cmax,
\)
and
\[
  m_-
  :=
  \cmin\cTmin
  \exp\!\big(-2\amin R_\star^2\big)
  \Big(\frac{\pi}{2\amin+\bmin}\Big)^{\!d/2}
\]
with $R_\star:=\|x_0\|+R_0.$ Where we used that fact
for any $\gamma>0,$
\begin{equation}
\label{gauss density formula}
\int_{\mathbb{R}^d}
  \exp(-\gamma\|z\|^2)\,dz
  =
  \left(\frac{\pi}{\gamma}\right)^{d/2}
  <\infty,
\end{equation}
and
\begin{equation}
\label{parallepiped inequaility}
    \|x - z\|^2 \le 2\|x\|^2 + 2\|z\|^2 \, .
\end{equation}
For each such $x,$ define a probability measure $\nu_x$ by
\[
  \nu_x(dz)
  :=
  \frac{q_T(x,z)\,\rho_T(z)}{m(x)}\,dz.
\]
Then
\[
  D_{g^\star}(x)
  =
  m(x)\,\mathbb{E}_{\nu_x}\Bigl[\frac{1}{g^\star(Z)}\Bigr].
\]
We now compare $\nu_x$ and $\rho_T$. The Radon--Nikodym derivative is
\[
  \frac{d\nu_x}{d\rho_T}(z)
  =
  \frac{q_T(x,z)}{m(x)}.
\]
By the two--sided Gaussian bounds in \emph{(Q)} and the bounds on $m$, there exist constants $0<k_-\le k_+<\infty$
such that
\[
  k_- \;\le\; \frac{d\nu_x}{d\rho_T}(z) \;\le\; k_+,
  \qquad
  \forall x\in\mathrm{supp}(\rho_0),\; z\in\mathbb{R}^d.
\]
For any nonnegative measurable $h$ we therefore have
\[
  k_- \int h(z)\,\rho_T(z)\,dz
  \;\le\;
  \int h(z)\,\frac{d\nu_x}{d\rho_T}(z)\,\rho_T(z)\,dz
  \;\le\;
  k_+ \int h(z)\,\rho_T(z)\,dz,
\]
that is,
\[
  k_-\,\mathbb{E}_{\rho_T}[h(Z)]
  \;\le\;
  \mathbb{E}_{\nu_x}[h(Z)]
  \;\le\;
  k_+\, \E_{\rho_T}[h(Z)].
\]
Apply this with $h(z) = \frac{1}{g^\star(z)}\ge 0$. Using the normalization in
\eqref{eq:fixed-point-normalization}, we obtain
\[
  k_- \le \mathbb{E}_{\nu_x}\Bigl[\frac{1}{g^\star(Z)}\Bigr] \le k_+.
\]
Consequently,
\begin{equation}
     \label{Dmin def}
     D_{g^\star}(x)
  =
  m(x)\,\mathbb{E}_{\nu_x}\Bigl[\frac{1}{g^\star(Z)}\Bigr]
  \ge
  m_-\,k_- =: \underline{D}^\star > 0
\end{equation}
for all $x\in\mathrm{supp}(\rho_0)$. Together with the upper bound, this proves
\eqref{eq:Dg-two-sided}.
Using the fixed--point equation and \eqref{eq:Dg-two-sided},
\[
  g^\star(y)
  =
  \int_{\mathrm{supp}(\rho_0)}
    \frac{\rho_0(x)}{D_{g^\star}(x)}\,q_T(x,y)\,dx
  \;\le\;
  \frac{1}{\underline{D}}
  \int_{B(x_0,R_0)} \rho_0(x)\,q_T(x,y)\,dx.
\]
 By the upper Gaussian bound in \emph{(Q)},
\[
  q_T(x,y)
  \le
  \cmax \exp\bigl(-\amax \|x-y\|^2\bigr).
\]
Using
\[
  \|x-y\|^2
  \ge \frac{1}{2}\|y\|^2 - \|x\|^2
  \ge \frac{1}{2}\|y\|^2 - R_\star^2,
\]
we obtain
\[
  q_T(x,y)
  \le
  \cmax \exp(\amax R_\star^2)\,\exp\!\left(-\frac{\amax}{2}\|y\|^2\right),
  \qquad x\in B(x_0,R_0).
\]
Therefore
\[
  g^\star(y)
  \le
  \frac{\cmax \exp(\amax R_\star^2)}{\underline{D}^\star}
  \int_{\mathrm{supp}(\rho_0)} \rho_0(x)\,dx\,
  \exp\!\left(-\frac{\amax}{2}\|y\|^2\right).
\]
Since $\rho_0$ is a probability density, $\int_{\mathrm{supp}(\rho_0)}\rho_0(x)\,dx = 1$, and we
may set
\[
  \cStarmax
  := \frac{\cmax \exp(\amax R_\star^2)}{\underline{D}^\star},
  \qquad
  \aStarmax := \frac{\amax}{2},
\]
to obtain
\[
  g^\star(y)
  \;\le\;
  \cStarmax \exp(- \aStarmax  \|y\|^2),
  \qquad \forall y\in \R^d.
\]
Using again the fixed--point equation and \eqref{eq:Dg-two-sided},
\[
  g^\star(y)
  =
  \int_{\mathrm{supp}(\rho_0)}
    \frac{\rho_0(x)}{D_{g^\star}(x)}\,q_T(x,y)\,dx
  \;\ge\;
  \frac{1}{\overline{D}^\star}
  \int_{\mathrm{supp}(\rho_0)} \rho_0(x)\,q_T(x,y)\,dx.
\]
By \emph{(R0)},
\[
  g^\star(y)
  \ge
  \frac{\rhomin}{\overline{D}^\star}
  \int_{B(x_0,r_0)} q_T(x,y)\,dx.
\]
Using the lower Gaussian bound in \emph{(Q)},
\[
  q_T(x,y) \ge \cmin \exp\bigl(-\amin \|x-y\|^2\bigr),
\]
and  $\|x\|\le \|x_0\|+r_0$ for $x\in B(x_0,r_0)$, we have
\[
  \|x-y\|
  \le \|x\| + \|y\|
  \le \|x_0\|+r_0 + \|y\|
\]
and therefore
\[
  \|x-y\|^2 \le (\|x_0\|+r_0+\|y\|)^2 \le 2 r_\star^2 + 2\|y\|^2,
\]
where $r_\star = \|x_0\|+r_0$. Thus,
\[
  q_T(x,y)
  \ge
  \cmin \exp\bigl(-\amin\,(2 r_\star^2 + 2\|y\|^2)\bigr),
  \qquad x\in B(x_0,r_0).
\]
It follows that
\[
  \int_{B(x_0,r_0)} q_T(x,y)\,dx
  \ge
  \cmin \exp(-2 \amin r_\star^2)\,\exp(-2 \amin \|y\|^2)\,\lambda^d(B(x_0,r_0)),
\]
where $\lambda^d(B(x_0,r_0))$ is the Lebesgue measure of $B(x_0,r_0)$.
Therefore
\[
  g^\star(y)
  \ge
  \frac{\rhomin\, \cmin\,\lambda^d(B(x_0,r_0))\,\exp(-2\amin r_\star^2)}{\overline{D}^\star}
  \exp(-2 \amin \|y\|^2).
\]
Setting
\[
  \cStarmin
  :=
  \frac{\rhomin\, \cmin\,\lambda^d(B(x_0,r_0))\,\exp(-2\amin r_\star^2)}{\overline{D}^\star},
  \qquad
  \aStarmin := 2 \amin,
\]
we obtain
\[
  g^\star(y)
  \;\ge\;
  \cStarmin \exp(-\aStarmin \|y\|^2),
  \qquad \forall y\in \R^d.
\]
Combining the upper and lower bounds completes the proof of
\eqref{eq:g-two-sided-Gaussian}.
\end{proof}

\section{Uniform concentration of the empirical risk}
\label{sec uniform concentration}

\begin{lemma}
\label{lemm: key properties}
Suppose that the assumptions \emph{(Q)}, \emph{(R0)}, \emph{(RT)} and \emph{(G)} hold.  Assume additionally that $\bmax> 2 \aGmin$. Then there exist constants $\underline{D}, \overline{D}, L_{D,\infty}, L_{\mathcal C, \infty} > 0$ such that
\begin{align}
\label{D bounds}
  & \underline{D}
  \;\le\;
  D_g(x)
  \;\le\;
  \overline{D},
  \qquad x\in B(x_0,R_0),\; g\in\mathcal{G}; \\
  \label{D Lip}
  & \|D_g-D_h\|_{L^\infty(B(x_0, R_0))}
  \le
  L_{D,\infty} \|g-h\|_{\infty};  \\
  \label{C lip infty}
  & \|\mathcal C[g]-\mathcal C[h]\|_\infty
  \le
  L_{\mathcal C, \infty}\,\|g-h\|_{L^\infty(\R^d)}; \\
  \label{C infty}
  & \|\mathcal C[g]\|_\infty
  \le
  (\cmax/\underline{D})\,\|g\|_{L^\infty(\R^d)}\, .
\end{align}
Note that $\underline{D}, \overline{D}$ do not depend on the class $\mathcal G$. 
\end{lemma}
\begin{proof}
The proof of \eqref{D bounds} repeats the proof of Proposition \ref{prop:g-two-sided-subgaussian-direct}. We make necessary changes. First, we replace \eqref{D upper boound} by
\begin{align*}
  D_{g}(x)
  &=
  \int q_T(x,z)\,\frac{\rho_T(z)}{g(z)}\,dz
  \le
  \cmax \cTmax \cGmin^{-1}\int 
  \exp\!\left(-\Bigl(\bmax- \aGmin\Bigr)\|z\|^2\right)\,dz
  \\
  & =
  \cmax \cTmax  \cGmin^{-1} \left(\frac{\pi}{\bmax- \aGmin}\right)^{d/2} = : \overline{D}\, .
\end{align*}
Recall the proof of \eqref{Dmin def}. Note that
$$
\E_{\rho_T}[1/g] \geq (\cGmax)^{-1}\,, 
$$
and we can take $\underline{D}: = m_-\,k_- (\cGmax)^{-1}$. 
We prove \eqref{D Lip}. By definition,
\[
  D_g(x)-D_h(x)
  =
  \int_{\mathbb{R}^d} q_T(x,z)\,\rho_T(z)\,
  \Bigl(\frac{1}{g(z)}-\frac{1}{h(z)}\Bigr)\,dz.
\]
Using
\[
  \left|\frac{1}{g(z)}-\frac{1}{h(z)}\right|
  =
  \frac{|g(z)-h(z)|}{g(z)h(z)}
  \le
  \|g-h\|_\infty\cdot \frac{1}{g(z)h(z)}
\]
and the lower bound in \emph{(G)} (applied to both $g$ and $h$),
we obtain for all $z$,
\[
  \frac{1}{g(z)h(z)}
  \le
  \cGmin^{-2}\exp(2 \aGmin \|z\|^2).
\]
Hence
\begin{equation}\label{eq:Dg-Dh-basic}
  |D_g(x)-D_h(x)|
  \le
  \|g-h\|_\infty\,
  (\cGmin)^{-2}
  \int_{\mathbb{R}^d}
    q_T(x,z)\,\rho_T(z)\,\exp(2 \aGmin \|z\|^2) dz.
\end{equation}
We now bound the integral uniformly in $x\in B(x_0, R_0)$. By \emph{(Q)} and the inequality
\[
  \|x-z\|^2 \ge \frac{1}{2}\|z\|^2 - \|x\|^2,
\]
we have, for all $x\in B(x_0, R_0)$,
\[
  q_T(x,z)
  \le
  \cmax \exp(-\amax\|x-z\|^2)
  \le \cmax \exp(\amax R_\star^2)\exp\!\left(-\frac{\amax}{2}\|z\|^2\right)
\]
with $R_\star:=\|x_0\|+R_0.$
Combining this with \emph{(RT)} yields
\[
  q_T(x,z)\,\rho_T(z)\,\exp(2 \aGmin \|z\|^2)
  \le
  \cmax
  \cTmax \exp(\amax R_\star^2)\,
  \exp\!\left(-\Bigl(\bmax- 2 \aGmin +\frac{\amax}{2}\Bigr)\|z\|^2\right).
\]
By \eqref{gauss density formula}, uniformly for $x\in B(x_0, R_0)$,
\[
  |D_g(x)-D_h(x)|
  \le
  L_{D,\infty} \|g-h\|_\infty \, ,
\]
where 
\begin{equation}
\label{LD definition}
L_{D,\infty}
  :=
  \cGmin^{-2}
  \cmax
  \cTmax \exp(\amax R_\star^2) 
  \left(\frac{\pi}{\bmax- 2 \aGmin +\amax/2}\right)^{d/2}\, .
\end{equation}
This proves \eqref{D Lip}. For each $y \in \R^d$,
\[
  \mathcal C[g](y)-\mathcal C[h](y)
  =
  \int_{\mathrm{supp}(\rho_0)}\rho_0(x)\,q_T(x,y)
  \left(\frac{1}{D_g(x)}-\frac{1}{D_h(x)}\right)\,dx.
\]
By \eqref{D bounds} and \eqref{D Lip} we get 
\begin{align*}
  |\mathcal C[g](y)-\mathcal C[h](y)|
  & \le
  \frac{L_{D,\infty} \|g-h\|_\infty}{\underline D^2}
  \int_{\mathrm{supp}(\rho_0)}\rho_0(x)\,q_T(x,y)\,dx \\
  & \le L_{\mathcal C, \infty} \|g-h\|_\infty, 
\end{align*}
where 
\begin{eqnarray}
\label{CD definition}
    L_{\mathcal C, \infty} = L_{D,\infty} \cmax \, .
\end{eqnarray}
This proves \eqref{C lip infty}. 

\end{proof}

\begin{proof}[Proof of Theorem \ref{thm:uniform-concentration-risk}]
Define an intermediate risk that uses the empirical average in $Y_j$ but
the population operator $\mathcal{C}[g]$:
\[
  \widetilde{\mathcal{R}}_M(g)
  :=
  \frac{1}{M}\sum_{j=1}^M
  \ell\bigl(g(Y_j),\mathcal{C}[g](Y_j)\bigr).
\]
Then
\[
  \widehat{\mathcal{R}}_{N,M}(g) - \mathcal{R}(g)
  =
  \underbrace{
    \bigl(\widehat{\mathcal{R}}_{N,M}(g) - \widetilde{\mathcal{R}}_M(g)\bigr)
  }_{(\mathrm{I})}
  +
  \underbrace{
    \bigl(\widetilde{\mathcal{R}}_M(g) - \mathcal{R}(g)\bigr)
  }_{(\mathrm{II})}.
\]
Taking supremum over $g\in\mathcal{G}$ and expectations yields
\[
  \mathbb{E}
  \sup_{g\in\mathcal{G}}
    \bigl|
      \widehat{\mathcal{R}}_{N,M}(g) - \mathcal{R}(g)
    \bigr|
  \;\le\;
  T_1 + T_2,
\]
where
\[
  T_1 := 
  \mathbb{E}
  \sup_{g\in\mathcal{G}}
    \bigl|
      \widehat{\mathcal{R}}_{N,M}(g) - \widetilde{\mathcal{R}}_M(g)
    \bigr|,
  \qquad
  T_2 :=
  \mathbb{E}
  \sup_{g\in\mathcal{G}}
    \bigl|
      \widetilde{\mathcal{R}}_M(g) - \mathcal{R}(g)
    \bigr|.
\]
For fixed $g$, $\widetilde{\mathcal{R}}_M(g)$ is the empirical average of the
i.i.d.\ variables
\[
  f_g(Y_j)
  :=
  \ell\bigl(g(Y_j),\mathcal{C}[g](Y_j)\bigr),
  \qquad j=1,\dots,M.
\]
Let
\(
  \mathcal{F}
  :=
  \bigl\{ f_g : g\in\mathcal{G}\bigr\}.
\)
By Lemma \ref{lemm: key properties}, 
\[
  |\mathcal{C}[g](y) - \mathcal{C}[g'](y)|
  \le
  L_{\mathcal C, \infty} \|g-g'\|_{L^\infty(\mathbb{R}^d)},\quad |\mathcal{C}[g](y)|\leq (\cmax/\underline{D})\|g\|_{L^\infty(\mathbb{R}^d)}.
\]
By  \emph{(L)} and boundedness of any $g\in\mathcal{G}$,
\[
  |f_g(y) - f_{g'}(y)|
  \le
  L_\ell\,\Bigl(
    |g(y) - g'(y)| + |\mathcal{C}[g](y) - \mathcal{C}[g'](y)|
  \Bigr)
\]
with $L_\ell=L_\ell\bigl(\cGmax(1+(\cmax/\underline{D})\bigr).$
Combining with \emph{(G)} and \emph{(RT)}, we obtain that $\mathcal{F}$ has 
envelope $F$ with  $\|F\|_{\infty}\leq B_\ell=B_\ell((1+(\cmax/\underline{D}))\,\cGmax)$).
Theorem~3.5.13 in \cite{GineNickl2016} implies
\begin{align}
\nonumber
    T_2 & \lesssim \frac{1}{\sqrt{M}} \int_{0}^{8 \|F\|_{\infty}} \sqrt{\log(2 N_{[]}(\mathcal{F}, L^{2}(\rho_T), \varepsilon))} \, d\varepsilon \\
    \label{eq:bound-T2}
    &\lesssim \frac{1}{\sqrt{M}}\int_{0}^{8  B_\ell} \sqrt{\log(2 N_{[]}(\mathcal{G}, L^{2}(\rho_T), \varepsilon/L_\ell))} \, d\varepsilon.
\end{align}
We now bound
\[
  T_1
  =
  \mathbb{E}
  \sup_{g\in\mathcal{G}}
    \bigl|
      \widehat{\mathcal{R}}_{N,M}(g) - \widetilde{\mathcal{R}}_M(g)
    \bigr|.
\]
By the Lipschitz property \emph{(L)}, for each $g$,
\begin{align*}
  \bigl|
    \widehat{\mathcal{R}}_{N,M}(g) - \widetilde{\mathcal{R}}_M(g)
  \bigr|
  &=
  \left|
    \frac{1}{M}\sum_{j=1}^M
      \Bigl(
        \ell\bigl(g(Y_j),\widehat{\mathcal{C}}_{N,M}[g](Y_j)\bigr)
        -
        \ell\bigl(g(Y_j),\mathcal{C}[g](Y_j)\bigr)
      \Bigr)
  \right|
  \\
  &\le
  \frac{L_\ell}{M}\sum_{j=1}^M
    \bigl|
      \widehat{\mathcal{C}}_{N,M}[g](Y_j) - \mathcal{C}[g](Y_j)
    \bigr|.
\end{align*}
Taking supremum over $g\in\mathcal{G}$ and expectations,
\[
  T_1
  \le
  L_\ell\,
  \mathbb{E}
  \left[
    \sup_{g\in\mathcal{G}}
      \frac{1}{M}\sum_{j=1}^M
        \bigl|
          \widehat{\mathcal{C}}_{N,M}[g](Y_j) - \mathcal{C}[g](Y_j)
        \bigr|
  \right].
\]
Fix a large compact set $K\subset\mathbb{R}^d$ (to be defined later) and consider 
\[
  \mathbb{E}
  \left[
    \sup_{g\in\mathcal{G}}\sup_{y\in K}
      \bigl|
        \widehat{\mathcal{C}}_{N,M}[g](y) - \mathcal{C}[g](y)
      \bigr|
  \right].
\]
By Lemma \ref{lemm: key properties}, 
\begin{equation}\label{eq:Dg-bounds-subgaussian}
  0 < \underline{D}
  \;\le\;
  D_g(x)
  \;\le\;
  \overline{D}
  < \infty,
  \quad x\in B(x_0, R_0),\quad g\in\mathcal{G}.
\end{equation}
Define
\[
  \widetilde{\mathcal{C}}_{N}[g](y)
  :=
  \frac{1}{N}\sum_{i=1}^N
     \frac{q_T(X_i,y)}{D_g(X_i)}.
\]
Then, for each $(g,y)$,
\[
  \widehat{\mathcal{C}}_{N,M}[g](y) - \mathcal{C}[g](y)
  =
  \underbrace{
    \bigl(
      \widehat{\mathcal{C}}_{N,M}[g](y) 
      - \widetilde{\mathcal{C}}_{N}[g](y)
    \bigr)
  }_{(\mathrm{A})}
  +
  \underbrace{
    \bigl(
      \widetilde{\mathcal{C}}_{N}[g](y) - \mathcal{C}[g](y)
    \bigr)
  }_{(\mathrm{B})}.
\]
Thus,
\[
  \sup_{g\in\mathcal{G}}\sup_{y\in \R^d}
    \bigl|
      \widehat{\mathcal{C}}_{N,M}[g](y) - \mathcal{C}[g](y)
    \bigr|
  \le
  \sup_{g\in\mathcal{G}}\sup_{y\in K}|(\mathrm{A})|
  +
  \sup_{g\in\mathcal{G}}\sup_{y\in K}|(\mathrm{B})|.
\]
For fixed $(g,y)$ define
\[
  \phi_{g,y}(x)
  :=
  \frac{q_T(x,y)}{D_g(x)}.
\]
Then
\[
  \widetilde{\mathcal{C}}_{N}[g](y)
  =
  \frac{1}{N}\sum_{i=1}^N \phi_{g,y}(X_i),
  \qquad
  \mathcal{C}[g](y)
  =
  \mathbb{E}_{X\sim\rho_0}[\phi_{g,y}(X)].
\]
 By \eqref{eq:Dg-bounds-subgaussian} and the upper Gaussian bound on $q_T$ in
\emph{(Q)}, together with compactness of $B(x_0, R_0)$, there exists $F_\infty<\infty$ (take $F_\infty=c_{+}/\underline{D}$) 
such that
\[
  |\phi_{g,y}(x)| \le F_\infty,
  \quad x\in B(x_0, R_0),\; y\in \R^d,\; g\in\mathcal{G}.
\]
 Fix $g,h\in\mathcal{G}$ and $x\in B(x_0, R_0)$. 
By Lemma \ref{lemm: key properties}, 
\[
  |D_g(x)-D_h(x)|
  \le
  L_{D,\infty} \|g-h\|_\infty\, .
\]
Hence, the map $(g,y)\mapsto \phi_{g,y}(x)$ is Lipschitz on $\mathcal{G}\times K$
with respect to $\|\cdot\|_\infty$ in $g$ and the Euclidean norm in $y$. For
the $y$--dependence this follows from smoothness (or Lipschitz continuity)
of $q_T$ and boundedness of $1/D_g(x)$ on $B(x_0, R_0)$.
Thus, for all $x\in B(x_0,R_0)$, $g,h\in\mathcal G$, and $y,y'\in K$,
\begin{equation}\label{eq:phi-Lip}
  |\phi_{g,y}(x)-\phi_{h,y'}(x)|
  \le
  L_g\,\|g-h\|_{L^\infty(\R^d)}
  +L_y\,\|y-y'\|,
\end{equation}
where
\begin{equation}\label{eq:Lg-Ly}
  L_g:=\frac{c_+\,L_{D,\infty}}{\underline D^{2}},
  \qquad
  L_y:=\frac{L_q}{\underline D}.
\end{equation}
Let
\[
  \Phi_K
  :=
  \bigl\{
    x\mapsto \phi_{g,y}(x)
   : g\in\mathcal{G}, y\in K
  \bigr\}.
\]
Note that
\begin{equation}\label{eq:bracketing-PhiK}
  \log \mathcal N_{[]}\!\big(\Phi_K,L^2(\rho_0), \varepsilon\big)
  \ \le\
  \log \mathcal N_{[]}\!\Big(\mathcal G,\|\cdot\|_\infty,\frac{\varepsilon}{2L_g}\Big)
  \ +\
  C_d\,\log\!\Big(\frac{L_y\,\mathrm{diam}(K)}{\varepsilon}\Big),
\end{equation}
where $C_d\le d$ is a dimensional constant (coming from covering $K$ in Euclidean norm).
Then it follows from Theorem~3.5.13 in \cite{GineNickl2016},
\begin{multline}
\label{eq:B-bound}
  \sup_{g\in\mathcal{G}}\sup_{y\in \R^d}|(\mathrm{B})| 
  \lesssim \frac{1}{\sqrt{N}} \int_{0}^{8 F_{\infty}} \sqrt{\log \mathcal N_{[]}\!\Big(\mathcal G,\|\cdot\|_\infty,\frac{\varepsilon}{2L_g}\Big)} \, d\varepsilon
   \\+\frac{1}{\sqrt{N}}\sqrt{\log\!\Big(\frac{L_y\,\mathrm{diam}(K)}{F_\infty}\Big)}.
\end{multline}
Furthermore, for a fixed $(g,y)$,
\[
  (\mathrm{A})
  =
  \frac{1}{N}\sum_{i=1}^N
    q_T(X_i,y)
    \left(
      \frac{1}{\widehat{D}_g(X_i)} - \frac{1}{D_g(X_i)}
    \right).
\]
By \eqref{eq:Dg-bounds-subgaussian}, $u\mapsto 1/u$ is Lipschitz on
$[\underline{D},\overline{D}]$ with constant
$L_D := \overline{D}^2/\underline{D}^2$, so
\[
  \left|
    \frac{1}{\mathcal{T}_{[\underline{D},\overline{D}]}[\widehat D_g](X_i)} - \frac{1}{D_g(X_i)}
  \right|
  \le
  L_D
  \bigl|
    \widehat{D}_g(X_i) - D_g(X_i)
  \bigr|.
\]
By \emph{(Q)} we have
\[
  q_T(x, y) \le \cmax \quad \text{for all } x, y \in\mathbb{R}^d.
\]
Hence, for all $X_i\in B(x_0, R_0)$ and all $y\in \R^d$,
\[
  q_T(X_i,y) \le \cmax.
\]
Thus, 
\[
  \sup_{y\in K}|(\mathrm{A})|
  \le
  \frac{\cmax L_D}{N}
  \sum_{i=1}^N
    \bigl|
      \widehat{D}_g(X_i) - D_g(X_i)
    \bigr|.
\]
Taking supremum over $g\in\mathcal{G}$ and expectations,
\[
  \mathbb{E}
  \left[
    \sup_{g\in\mathcal{G}}\sup_{y\in K}|(\mathrm{A})|
  \right]
  \le
  \cmax L_D\,
  \mathbb{E}
  \left[
    \sup_{g\in\mathcal{G}}
      \frac{1}{N}\sum_{i=1}^N
        \bigl|
          \widehat{D}_g(X_i) - D_g(X_i)
        \bigr|
  \right].
\]
Now for each fixed $x$ and $g$,
\[
  D_g(x)
  =
  \mathbb{E}_{Z\sim\rho_T}\bigl[\psi_g(x,Z)\bigr],
  \qquad
  \widehat{D}_g(x)
  =
  \frac{1}{M}\sum_{k=1}^M \psi_g(x,Y_k),
\]
with
\[
  \psi_g(x,z)
  :=
  q_T(x,z)\,\frac{1}{g(z)}.
\]
Let
\[
  \Psi
  :=
  \bigl\{
    (x,z)\mapsto \psi_g(x,z)
   : x\in B(x_0, R_0),\;g\in\mathcal{G}
  \bigr\}.
\]
Due to \emph{(G)}, 
\[
  g(z)\ge \cGmin e^{-\aGmin\|z\|^2},\qquad \forall z\in\R^d,
\]
and we have  for all $x\in B(x_0,R_0)$ and all $g\in\mathcal G$,
\[
  \psi_g(x,z)\le F(z):=\frac{c_+}{\cGmin}\,e^{\aGmin\|z\|^2}.
\]
Since  $\rho_T$ satisfies the Gaussian upper bound (RT) with $\bmax>2\aGmin,$
we get $F\in L^2(\rho_T)$ and
\[
  \|F\|_{L^2(\rho_T)}
  \le
  \frac{c_+}{\cGmin}\,\sqrt{\cTmax}\,
  \Big(\frac{\pi}{b_T-2\aGmin}\Big)^{d/4}.
\]
Theorem~3.5.13 in \cite{GineNickl2016} gives
\begin{align}
\E\Big[\sup_{x\in B(x_0,R_0)}\sup_{g\in\mathcal G}
  \big|\widehat D_g(x)-D_g(x)\big|\Big]
\lesssim \frac{1}{\sqrt{M}} \int_{0}^{8 \|F\|_{L^2(\rho_T)}} \sqrt{\log(2 N_{[]}(\Psi, L^{2}(\rho_T), \varepsilon))} \, d\varepsilon.
\end{align}
Fix $x\in B(x_0,R_0)$ and $z\in\R^d$. Write
\[
  \psi_g(x,z)-\psi_h(x,z)
  =
  q_T(x,z)\Big(\frac1{g(z)}-\frac1{h(z)}\Big).
\]
We have
\[
  \frac1{g(z)}\le \frac1{\cGmin}e^{\aGmin\|z\|^2},
  \qquad
  \frac1{h(z)}\le \frac1{\cGmin}e^{\aGmin\|z\|^2}.
\]
Hence, using $|1/u-1/v|=|u-v|/(uv)$,
\[
  \Big|\frac1{g(z)}-\frac1{h(z)}\Big|
  \le
  \frac{|g(z)-h(z)|}{\cGmin^2}\,e^{2\aGmin\|z\|^2}.
\]
Multiplying by $|q_T(x,z)|\le \cmax$ gives
\[
  |\psi_g(x,z)-\psi_h(x,z)|
  \le
  \frac{\cmax}{\cGmin^2}\,e^{2\aGmin\|z\|^2}\,\|g-h\|_\infty.
\]
Taking $L^2(\rho_T)$ norms yields
\[
  \|\psi_g(x,\cdot)-\psi_h(x,\cdot)\|_{L^2(\rho_T)}
  \le
  \frac{\cmax}{\cGmin^2}\,
  \Big(\int e^{4\aGmin\|z\|^2}\rho_T(z)\,dz\Big)^{1/2}\,
  \|g-h\|_\infty.
\]
Hence,
\begin{equation*}\label{eq:Psi-L2-Lip-infty}
  \|\psi_g(x,\cdot)-\psi_h(x,\cdot)\|_{L^2(\rho_T)}
  \ \le\
  \frac{\cmax}{\cGmin^2}\,M_\rho\,\|g-h\|_\infty
\end{equation*}
and as a result we have ,
\[
  \log \mathcal N_{[]}\!\big(\Psi,L^2(\rho_T), \varepsilon\big)
  \ \le\
  \log \mathcal N_{[]}\!\Big(\mathcal G,\|\cdot\|_\infty,\frac{\varepsilon}{C_g}\Big)
  \;+\;
  d\,\log\!\Big(\frac{C\,R_0}{\varepsilon}\Big),
\]
where \( C_g:=\frac{\cmax}{\cGmin^2}\,M_\rho.\)
We choose $K = B(0, R)$ with $R>R_0.$ Fix $g\in\mathcal G$ and $y\in\R^d$ with $\|y\|\ge R$. For any $x\in B(0,R_0),$ we have
\[
  \|x-y\|\ \ge\ \|y\|-\|x\|\ \ge\ R-R_0,
\]
hence by the kernel upper bound
\[
  q_T(x,y)\le \cmax \exp\!\bigl(-\amax\|x-y\|^2\bigr)
  \le \cmax \exp\!\bigl(-\amax (R-R_0)^2\bigr).
\]
Using $D_g(x)\ge \underline D$ and that $\rho_0$ is a probability density,
\begin{align*}
  |\mathcal C[g](y)|
  &=
  \Big|\int_{B(0,R_0)} \rho_0(x)\,\frac{q_T(x,y)}{D_g(x)}\,dx\Big|
  \le
  \frac{1}{\underline D}\int_{B(0,R_0)}\rho_0(x)\,q_T(x,y)\,dx\\
  &\le
  \frac{1}{\underline D}\sup_{x\in B(0,R_0)}q_T(x,y)
  \le
  \frac{\cmax}{\underline D}\exp\!\bigl(-\amax (R-R_0)^2\bigr).
\end{align*}
Furthermore, using $\mathcal{T}_{[\underline{D},\overline{D}]}[\widehat D_g](X_i)\ge \underline D$ and $X_i\in B(0,R_0)$,
\begin{align*}
  |\widehat{\mathcal C}_{N,M}[g](y)|
  &\le
  \frac{1}{N}\sum_{i=1}^N \frac{q_T(X_i,y)}{\mathcal{T}_{[\underline{D},\overline{D}]}[\widehat D_g](X_i)}
  \le
  \frac{1}{\underline D}\sup_{x\in B(0,R_0)} q_T(x,y)
  \le
  \frac{\cmax}{\underline D}\exp\!\bigl(-\amax (R-R_0)^2\bigr).
\end{align*}
Therefore, for all $g\in\mathcal G$ and all $\|y\|\ge R$,
\[
  \big|\widehat{\mathcal C}_{N,M}[g](y)-\mathcal C[g](y)\big|
  \le
  |\widehat{\mathcal C}_{N,M}[g](y)|+|\mathcal C[g](y)|
  \le
  \frac{2\cmax}{\underline D}\exp\!\bigl(-\amax (R-R_0)^2\bigr).
\]
Then for every $R>R_0$,
\begin{align}
\mathbb{E}\Bigg[
  \sup_{g\in\mathcal{G}}\sup_{y\not\in B(0,R)}
      \bigl|
        \widehat{\mathcal{C}}_{N,M}[g](y) - \mathcal{C}[g](y)
      \bigr|
\Bigg]
&\le
\frac{2\cmax}{\underline D}\,
\exp\!\bigl(-\amax (R-R_0)^2\bigr).
\label{eq:tail-outside-ball}
\end{align}
 
\end{proof}

\section{Approximation error}

\begin{lemma}
\label{prop:C-L2-Lipschitz}
Suppose that the assumptions \emph{(Q)}, \emph{(R0)}, \emph{(RT)} and \emph{(G)} hold.  Assume additionally that $\bmax >  4 \aGmin$. Then there exists constant $L_{\mathcal C, 2} > 0 $ such that
\begin{align}
   \label{C lip L^2}
&\|\mathcal C[g]-\mathcal C[h]\|_{L^2(\rho_T)}
  \le
  L_{\mathcal C, 2}\,\|g-h\|_{L^2(\rho_T)}.
\end{align}
\end{lemma}
\begin{proof}
Similarly to \eqref{eq:Dg-Dh-basic},
\[
  |D_g(x)-D_h(x)|
  \le
  (\cGmin)^{-2}\int q_T(x,z)\,\rho_T(z)\,\exp(2 \aGmin \|z\|^2)\,|g(z)-h(z)|\,dz.
\]
Apply Cauchy--Schwarz w.r.t.\ $\rho_T(z)\,dz$:
\[
  |D_g(x)-D_h(x)|
  \le
  (\cGmin)^{-2}
  \Bigl(\int q^2_T(x-z) \exp(4 \aGmin \|z\|^2)\,\rho_T(z)\,dz\Bigr)^{1/2}
  \|g-h\|_{L^2(\rho_T)}.
\]
Taking the supremum over $x\in B(x_0, R_0)$ yields
\begin{equation}\label{eq:D-L2}
  \|D_g-D_h\|_{L^\infty(B(x_0, R_0))}
  \le
  L_{D, 2} \,\|g-h\|_{L^2(\rho_T)},
\end{equation}
for some constant \(L_{D,2}<\infty\) provided 
\[
  \sup_{x\in B(x_0, R_0)}\int q_T^2(x,z) e^{4 \aGmin\|z\|^2}\,\rho_T(z)\,dz
\]
is finite. We now check this finiteness under \emph{(Q)} and the upper tail in \emph{(RT)}.
By  \emph{(Q)},
\[
  q_T^2(x,z) \le \cmax^2 \exp(-2\amax\|x-z\|^2)
  \le \cmax^2 \exp(2 \amax R_\star^2)\exp(-\amax\|z\|^2),
\]
We obtain
\[
  q_T^2(x,z) e^{4 \aGmin \|z\|^2}
  \le
  \cmax^2  \exp(2 \amax R_\star^2)\,
  \exp\!\Bigl(-\bigl(\amax  - 4 \aGmin + \bmax \bigr)\|z\|^2\Bigr).
\]
This is integrable if $\amax + \bmax - 4 \aGmin>0$. 
Hence $L_{D,2}<\infty$ and \eqref{eq:D-L2} holds. By \eqref{eq:D-L2} it holds that 
\[
  |\mathcal C[g](y)-\mathcal C[h](y)|
  \le
  \frac{L_{D,2}}{\underline D^2}
  \|g-h\|_{L^2(\rho_T)}
  \underbrace{\int_{K_0}\rho_0(x)\,q_T(x,y)\,dx}_{=:m_0(y)}.
\]
Now take $L^2(\rho_T)$ norms in $y$:
\[
  \|\mathcal C[g]-\mathcal C[h]\|_{L^2(\rho_T)}
  \le
  \frac{L_{D,2}}{\underline D^2}\,\|m_0\|_{L^2(\rho_T)}\,\|g-h\|_{L^2(\rho_T)}.
\]
Finally, $\|m_0\|_{L^2(\rho_T)}<\infty$ under \emph{(Q)},\emph{(R0)},\emph{(RT)} since $m_0(y)$ is
sub-Gaussian in $y$ (a compactly supported mixture of sub-Gaussian kernels) and
$\rho_T$ has sub-Gaussian tails. Therefore \eqref{C lip L^2} holds with
\[
  L_{\mathcal C,2}
  :=
  \frac{L_{D,2}}{\underline D^2}\,\|m_0\|_{L^2(\rho_T)} <\infty.
\]
\end{proof}

\begin{corollary}
\label{cor:excess-risk-L2}
Let $g^\star\in\mathcal G$ be a fixed point of $\mathcal C$, i.e.
$g^\star=\mathcal C[g^\star]$. Then under assumptions of Lemma~\ref{prop:C-L2-Lipschitz}, it holds for all $g\in\mathcal G$,
\begin{equation}\label{eq:excess-risk-bound}
  \big|\mathcal R(g)-\mathcal R(g^\star)\big|
  \le
  L_\ell\bigl(1+L_{\mathcal C,2}\bigr)\,\|g-g^\star\|_{L^2(\rho_T)}.
\end{equation}

\end{corollary}
\begin{proof}
Using \emph{(L)} and the fixed point property $\mathcal C[g^\star]=g^\star$, we have
almost surely (with $Y\sim\rho_T$)
\[
  \Big|\ell\big(g(Y),\mathcal C[g](Y)\big)-\ell\big(g^\star(Y),g^\star(Y)\big)\Big|
  \le
  L_\ell\Big(|g(Y)-g^\star(Y)|+|\mathcal C[g](Y)-g^\star(Y)|\Big).
\]
Taking expectations and applying Cauchy--Schwarz gives
\[
  \big|\mathcal R(g)-\mathcal R(g^\star)\big|
  \le
  L_\ell\Big(\|g-g^\star\|_{L^2(\rho_T)}+\|\mathcal C[g]-g^\star\|_{L^2(\rho_T)}\Big).
\]
Since $g^\star=\mathcal C[g^\star]$,
\[
  \|\mathcal C[g]-g^\star\|_{L^2(\rho_T)}
  =
  \|\mathcal C[g]-\mathcal C[g^\star]\|_{L^2(\rho_T)}
  \le
  L_{\mathcal C, 2}\,\|g-g^\star\|_{L^2(\rho_T)}
\]
by Lemma~\ref{prop:C-L2-Lipschitz}. Combining these inequalities yields
\eqref{eq:excess-risk-bound}.
\end{proof}

\subsection{Bargmann transform}

For $f\in L^2(\R^d)$, the (Segal-)Bargmann transform is defined by
\[
(\mathcal B f)(z)
:=\pi^{-d/4}\int_{\R^d}
\exp\!\Big(-\tfrac12|y|^2+\sqrt2\,z\cdot y-\tfrac12\,z\cdot z\Big)\,f(y)\,dy,
\qquad z\in\C^d,
\]
where $z\cdot y=\sum_{j=1}^d z_j y_j$ and $z\cdot z=\sum_{j=1}^d z_j^2$.
The map $\mathcal B$ is unitary from $L^2(\R^d)$ onto the Fock space
$\mathcal F^2(\C^d)$ of entire functions. Define the  Hermite polynomials by
\[
\mathrm{He}_n(x):=(-1)^n e^{x^2}\frac{d^n}{dx^n}\Big(e^{-x^2}\Big),
\qquad n\in\N_0.
\]
Then the corresponding $L^2(\R)$--normalized Hermite functions are
\begin{equation}\label{eq:psi-n-prob}
\psi_n(x)
:=\frac{1}{2^{n/2}\pi^{1/4}\sqrt{n!}}\,
\mathrm{He}_n(x)\,e^{-x^2/2},
\qquad x\in\R,\ \ n\in\N_0.
\end{equation}
A direct computation (using the generating function of Hermite polynomials)
shows that for all $n\in\N_0$,
\[
(\mathcal B\psi_n)(z)=\frac{z^n}{\sqrt{n!}},
\qquad z\in\C.
\]
Indeed, by definition,
\[
(\mathcal B\psi_0)(z)
=\pi^{-1/4}\int_{\R}
e^{-\frac12 y^2+\sqrt2zy-\frac12 z^2}\;\pi^{-1/4}e^{-y^2/2}\,dy
=\pi^{-1/2}\int_{\R}e^{-y^2+\sqrt2zy-\frac12 z^2}\,dy.
\]
Complete the square:
\[
-y^2+\sqrt2zy=-\big(y-\tfrac{z}{\sqrt2}\big)^2+\tfrac12 z^2,
\]
hence
\[
-y^2+\sqrt2zy-\tfrac12 z^2=-\big(y-\tfrac{z}{\sqrt2}\big)^2.
\]
Therefore,
\[
(\mathcal B\psi_0)(z)
=\pi^{-1/2}\int_{\R}e^{-(y-\frac{z}{\sqrt2})^2}\,dy
=\pi^{-1/2}\int_{\R}e^{-u^2}\,du
=1.
\]
Let $f\in\mathcal S(\R)$ (Schwartz) so that all integrations by parts are justified.
Write
\[
K(y,z):=\exp\!\Big(-\tfrac12 y^2+\sqrt2zy-\tfrac12 z^2\Big).
\]
Then
\[
(\mathcal B(a^\ast f))(z)
=\pi^{-1/4}\int_{\R}K(y,z)\,\frac{1}{\sqrt2}\Big(y-\frac{d}{dy}\Big)f(y)\,dy
=\frac{1}{\sqrt2}\pi^{-1/4}\Big(I_1-I_2\Big),
\]
where
\[
I_1:=\int_{\R}K(y,z)\,y\,f(y)\,dy,
\qquad
I_2:=\int_{\R}K(y,z)\,f'(y)\,dy
\]
and
\[
a^\ast:=\frac{1}{\sqrt2}\Big(y-\frac{d}{dy}\Big).
\]
Integrate by parts in $I_2$ (boundary terms vanish since $K(\cdot,z)$ has Gaussian decay and
$f$ is Schwartz):
\[
I_2=-\int_{\R}\partial_y K(y,z)\,f(y)\,dy.
\]
Compute $\partial_y K$:
\[
\partial_y K(y,z)=(-y+\sqrt2 z)\,K(y,z).
\]
Hence
\[
I_2=-\int_{\R}(-y+\sqrt2 z)\,K(y,z)\,f(y)\,dy
=\int_{\R}y\,K(y,z)\,f(y)\,dy-\sqrt2 z\int_{\R}K(y,z)\,f(y)\,dy.
\]
Therefore,
\[
I_1-I_2
=\sqrt2 z\int_{\R}K(y,z)\,f(y)\,dy,
\]
and so
\[
(\mathcal B(a^\ast f))(z)
=\frac{1}{\sqrt2}\pi^{-1/4}\cdot\sqrt2 z\int_{\R}K(y,z)\,f(y)\,dy
=z\,(\mathcal B f)(z).
\]
Thus we have shown
\begin{equation}\label{eq:intertwine}
\mathcal B\circ a^\ast = (\text{multiplication by }z)\circ \mathcal B
\qquad\text{on }\mathcal S(\R).
\end{equation}
By the definition of $\psi_n$ we have \(\psi_n:=\frac{1}{\sqrt{n}}\;a^\ast\psi_{n-1},\)  \(n\ge1\) and \eqref{eq:intertwine} implies
\[
(\mathcal B\psi_n)(z)
=\frac{1}{\sqrt{n}}\,(\mathcal B(a^\ast\psi_{n-1}))(z)
=\frac{1}{\sqrt{n}}\,z\,(\mathcal B\psi_{n-1})(z).
\]
Starting from $(\mathcal B\psi_0)(z)=1$, we obtain recursively
\[
(\mathcal B\psi_n)(z)
=\frac{z^n}{\sqrt{n!}},
\qquad n\in\N_0.
\]
This identity reflects the fact that the Bargmann transform diagonalizes
the harmonic oscillator: Hermite functions are mapped to monomials.
The $d$--dimensional Hermite functions factorize as
\[
\psi_\alpha(x)=\prod_{j=1}^d \psi_{\alpha_j}(x_j),
\qquad \alpha=(\alpha_1,\dots,\alpha_d)\in\N_0^d.
\]
Likewise, the Bargmann kernel factorizes coordinatewise:
\[
e^{-\frac12|y|^2+\sqrt2\,z\cdot y-\frac12 z\cdot z}
=\prod_{j=1}^d
e^{-\frac12 y_j^2+\sqrt2\,z_j y_j-\frac12 z_j^2}.
\]
Applying Fubini's theorem and the one--dimensional identity yields
\begin{equation}
\label{eq:Bpsi}
 (\mathcal B\psi_\alpha)(z)
=\prod_{j=1}^d (\mathcal B\psi_{\alpha_j})(z_j)
=\prod_{j=1}^d \frac{z_j^{\alpha_j}}{\sqrt{\alpha_j!}}
=\frac{z^\alpha}{\sqrt{\alpha!}},   
\end{equation}
where $z^\alpha=\prod_{j=1}^d z_j^{\alpha_j}$ and
$\alpha!=\prod_{j=1}^d \alpha_j!$. 
Since $\{\psi_\alpha\}_{\alpha\in\N_0^d}$ is an orthonormal basis of $L^2(\R^d)$,
every $f\in L^2(\R^d)$ admits the Hermite expansion
\[
f=\sum_{\alpha\in\N_0^d}\langle f,\psi_\alpha\rangle\,\psi_\alpha
\quad\text{(in }L^2(\R^d)\text{)}.
\]
Because $\mathcal B$ is unitary, we may apply it termwise:
\[
\mathcal B f
=\sum_{\alpha\in\N_0^d}\langle f,\psi_\alpha\rangle\,\mathcal B\psi_\alpha.
\]
Using \eqref{eq:Bpsi}, we obtain the power--series representation
\[
(\mathcal B f)(z)
=\sum_{\alpha\in\N_0^d}\langle f,\psi_\alpha\rangle\,\frac{z^\alpha}{\sqrt{\alpha!}},
\qquad z\in\C^d.
\]
Thus, the Bargmann transform converts the Hermite expansion of $f$
into the Taylor expansion of the entire function $\mathcal B f$.
In this correspondence,
\[
\langle f,\psi_\alpha\rangle
=\frac{\partial^\alpha(\mathcal B f)(0)}{\sqrt{\alpha!}},
\]
which explains why bounds on the growth of $\mathcal B f$
immediately yield decay estimates for Hermite coefficients.

\subsection{Approximation by Hermite polynomials}
Since we are interested in the case $T = 1$ we omit $T$ from the notation of $q$. Then
$$
q(z):=(2\pi)^{-d/2}\exp\!\Big(-\frac{\|z\|^2}{2}\Big)\, .
$$

\begin{proposition}
\label{prop:hermite-coeff}
Let $d\ge 1$ let $w:\R^d\to[0,\infty)$ be compactly supported with
\[
\mathrm{supp}(w)\subseteq \{x\in\R^d:\ \|x\|\le R\}, 
\qquad 
M_0:=\int_{\R^d} w(x)\,dx<\infty .
\]
Define
\[
g^\star(y):=(w*q)(y)=\int_{\R^d} w(x)\,q(x-y)\,dx .
\]
Let $\{\psi_\alpha\}_{\alpha\in\N_0^d}$ be the $d$-dimensional tensor-product Hermite basis orthonormal in $L^2(\R^d)$
and define the Hermite coefficients
\[
c_\alpha:=\langle g^\star,\psi_\alpha\rangle_{L^2(\R^d)} ,\qquad \alpha\in\N_0^d.
\]
Then for every multi-index $\alpha\in\N_0^d$ with $m:=|\alpha|\ge 1$,
\begin{equation}\label{eq:coeff-bound-T1-clean}
|c_\alpha|
\ \le\
C_d\,
m^{1/4}\,
\Big(\frac{e^{1/2}\beta}{\sqrt m}\Big)^{m},
\qquad 
C_d=\pi^{-d/4}2^{-d/2}M_0,
\qquad
\beta:=R\sqrt{\frac d2}.
\end{equation}
Equivalently,
\begin{equation}\label{eq:coeff-bound-T1-logform}
|c_\alpha|
\ \le\
C_d\,
m^{1/4}\,
\exp\!\Big(-\frac m2\log m + m\log(e^{1/2}\beta)\Big).
\end{equation}
In particular, if $m\ge (e^{1/2}\beta)^2=e\beta^2$, then $\log(e^{1/2}\beta)\le \tfrac12\log m$ and 
\begin{equation}\label{eq:coeff-supergeom-T1}
|c_\alpha|
\ \le\
C_d\,
m^{1/4}\,
\exp\!\Big(-\frac m4\log m\Big).
\end{equation}
\end{proposition}
\begin{remark}\label{rem:choice-lambda-linear-growth}
Consider the Gaussian kernel with variance $T>1$,
\[
q_T(x-y)=(2\pi T)^{-d/2}\exp\!\Big(-\frac{|x-y|^2}{2T}\Big),
\qquad g^\star=w*q_T .
\]
Although the unscaled Bargmann transform of $g^\star$ exhibits Gaussian
(quadratic) growth in the complex variable $z$, this can be compensated by a
suitable scaling of the Hermite basis.
More precisely, let $\lambda\in(0,1)$ be chosen such that
\[
\lambda^2=\frac{1}{T},
\qquad\text{equivalently}\qquad
\lambda=T^{-1/2}.
\]
If one expands $g^\star$ in the scaled Hermite basis
\[
\psi_\alpha^{(1/\lambda)}(x)
=\lambda^{d/2}\prod_{j=1}^d \psi_{\alpha_j}(\lambda x_j),
\qquad \alpha\in\N_0^d,
\]
and considers the associated scaled Bargmann transform, then the quadratic term
in the exponential prefactor of the Bargmann representation is exactly
neutralized by the scaling.
As a result, the Bargmann transform of $g^\star$ in the scaled variables obeys a
\emph{linear--exponential} growth bound of the form
\[
\sup_{\max_j|z_j|\le r}
\big|(\mathcal B g^\star)(z)\big|
\;\le\;
C_{d,T}\,
\exp\!\big(b_{d,T}\,r\big),
\qquad r>0,
\]
with constants $C_{d,T},b_{d,T}>0$ depending only on $d$, $T$, and the support
radius of $w$, but \emph{not} on $r$.
In particular, under this proper choice of the scaling parameter $\lambda$, the
Bargmann transform exhibits the same linear-exponential growth as in the
critical case $T=1$, and the resulting scaled Hermite coefficients decay
super-geometrically in the total degree $|\alpha|$. This observation explains
why the scaling $\lambda=T^{-1/2}$ is natural and optimal for extending the
$T=1$ coefficient estimates to general variances $T>1$.
\end{remark}

\begin{proof}
The Bargmann transform maps Hermite functions to normalized monomials:
\[
(\mathcal B\psi_\alpha)(z)=\frac{z^\alpha}{\sqrt{\alpha!}},\qquad \alpha\in\N_0^d,
\]
and therefore, for $g^\star\in L^2(\R^d)$,
\[
(\mathcal B g^\star)(z)=\sum_{\alpha\in\N_0^d} c_\alpha\,\frac{z^\alpha}{\sqrt{\alpha!}},
\qquad z\in\C^d.
\]
In particular,
\begin{equation}\label{eq:coeff-derivative}
c_\alpha=\frac{\partial^\alpha(\mathcal B g^\star)(0)}{\sqrt{\alpha!}}.
\end{equation}
For $r>0$ set the closed polydisk $D_r:=\{z\in\C^d:\ \max_j|z_j|\le r\}$.
By the multivariate Cauchy estimate,
\[
|\partial^\alpha F(0)|
\le \alpha!\,r^{-|\alpha|}\,\sup_{z\in D_r}|F(z)|,
\qquad \alpha\in\N_0^d.
\]
Apply this with $F=\mathcal B g^\star$ and combine with \eqref{eq:coeff-derivative}:
\begin{equation}\label{eq:cauchy-coeff}
|c_\alpha|
\le
\sqrt{\alpha!}\,r^{-m}\,\sup_{z\in D_r}|(\mathcal B g^\star)(z)|,
\qquad m:=|\alpha|.
\end{equation}
Insert \eqref{eq:bargmann-polydisk-T1} (see Proposition \ref{prop:bargmann-bound}) into \eqref{eq:cauchy-coeff}:
\[
|c_\alpha|
\le
C_0\,\sqrt{\alpha!}\,r^{-m}e^{\beta r},
\qquad r>0.
\]
Minimize $\phi(r):=\beta r-m\log r$ over $r>0$. Since
$\phi'(r)=\beta-\frac{m}{r}$, the unique minimizer is
\[
r_\ast=\frac{m}{\beta}.
\]
Hence
\begin{equation}\label{eq:coeff-after-opt}
|c_\alpha|
\le
C_0\,\sqrt{\alpha!}\,
\Big(\frac{\beta}{m}\Big)^m
e^{\beta(m/\beta)}
=
C_0\,\sqrt{\alpha!}\,
\Big(\frac{e\beta}{m}\Big)^m.
\end{equation}
Since $\alpha!\le m!$ for $m=|\alpha|$, we have $\sqrt{\alpha!}\le \sqrt{m!}$.
By Stirling's estimate there exists an absolute constant $C>0$ such that
\begin{equation}\label{eq:stirling-sqrtm}
\sqrt{m!}\ \le\ C\,m^{1/4}\Big(\frac{m}{e}\Big)^{m/2},
\qquad m\ge 1.
\end{equation}
Combining \eqref{eq:coeff-after-opt} and \eqref{eq:stirling-sqrtm} yields
\[
|c_\alpha|
\le
C\,C_0\,m^{1/4}\Big(\frac{m}{e}\Big)^{m/2}\Big(\frac{e\beta}{m}\Big)^m.
\]
Simplify:
\[
\Big(\frac{m}{e}\Big)^{m/2}\Big(\frac{e\beta}{m}\Big)^m
=
m^{m/2}e^{-m/2}\cdot e^m\beta^m m^{-m}
=
e^{m/2}\beta^m m^{-m/2}
=
\Big(\frac{e^{1/2}\beta}{\sqrt m}\Big)^m.
\]
Therefore
\[
|c_\alpha|
\le
C\,C_0\,
m^{1/4}\,
\Big(\frac{e^{1/2}\beta}{\sqrt m}\Big)^m,
\]
which is \eqref{eq:coeff-bound-T1-clean}.
\end{proof}
\begin{proposition}\label{prop:proj-error-T1-unscaled}
Let $d\ge 1$ and let $\{\psi_\alpha\}_{\alpha\in\N_0^d}$ be the  $d$--dimensional Hermite basis in
$L^2(\R^d)$. Let $\Pi^F_n$ denote the $L^2(\R^d)$--orthogonal projector onto
\[
\mathrm{span}\{\psi_\alpha:\ |\alpha|\le n\}.
\]
Under assumptions of Proposition~\ref{prop:hermite-coeff}  for every $n\in\N_0$ with
\(n+1>2K^2,\)
one has
\begin{equation}\label{eq:proj-error-T1-unscaled}
\|g^\star-\Pi^F_n g^\star\|_{L^2(\R^d)}
\ \le\
\widetilde C_d\,C_d\,
(n+1)^{\frac d2-\frac14}\,
\Big(\frac{K}{\sqrt{n+1}}\Big)^{n+1}
\end{equation}
where $C_d$ is given in \eqref{eq:coeff-bound-T1-clean} and 
\[
\widetilde C_d= \left(1 + 2^{d-1/2}\,\frac{\Gamma(d+1/2)}{(\ln 2)^{d+1/2}}\right)^{1/2}.
\]
\end{proposition}

\begin{proof}
By orthonormality of $\{\psi_\alpha\}$ and the definition of $\Pi^F_n$,
\begin{equation}\label{eq:l2-tail-unscaled}
\|g^\star-\Pi^F_n g^\star\|_{L^2(\R^d)}^2
=
\sum_{|\alpha|>n}|c_\alpha|^2.
\end{equation}
Using \eqref{eq:coeff-bound-T1-clean} with $m:=|\alpha|$ gives
\[
\sum_{|\alpha|>n}|c_\alpha|^2
\le
C_d^2\sum_{|\alpha|>n} m^{1/2}\Big(\frac{K^2}{m}\Big)^m.
\]
Group by total degree $m=|\alpha|$ and let
\[
N_d(m):=\#\{\alpha\in\N_0^d:\ |\alpha|=m\}=\binom{m+d-1}{d-1}.
\]
Then
\begin{equation}\label{eq:group-unscaled}
\sum_{|\alpha|>n}|c_\alpha|^2
\le
C_d^2\sum_{m=n+1}^\infty N_d(m)\, m^{1/2}\Big(\frac{K^2}{m}\Big)^m.
\end{equation}
Using the standard estimate $N_d(m)\le \widetilde C_d\,m^{d-1}$,
we obtain
\begin{equation}\label{eq:poly-unscaled}
\sum_{|\alpha|>n}|c_\alpha|^2
\le
\widetilde C_d\,C_d^2\sum_{m=n+1}^\infty m^{d-\frac12}\Big(\frac{K^2}{m}\Big)^m.
\end{equation}
Set
\[
q:=\frac{K^2}{n+1}\in(0,1).
\]
For $m\ge n+1$ we have $\frac{K^2}{m}\le q$, hence
\[
\Big(\frac{K^2}{m}\Big)^m \le q^m.
\]
Therefore,
\begin{eqnarray*}
\sum_{m=n+1}^\infty m^{d-\frac12}\Big(\frac{K^2}{m}\Big)^m
&\le &
q^{n+1}(n+1)^{d-1/2}
\sum_{k=0}^\infty \Bigl(1+\frac{k}{n+1}\Bigr)^{d-1/2} q^k
\\
&\le & q^{n+1}(n+1)^{d-1/2}\left(1 + 2^{d-1/2}\,\frac{\Gamma(d+1/2)}{(-\ln q)^{d+1/2}}\right)
\end{eqnarray*}
Insert this into \eqref{eq:poly-unscaled} and take square roots:
\[
\|g^\star-\Pi^F_n g^\star\|_{2}
\le
\widetilde C_d\,C_d\,
(n+1)^{\frac d2-\frac14}\,q^{(n+1)/2}.
\]
Finally, $q^{(n+1)/2}=\Big(\frac{K}{\sqrt{n+1}}\Big)^{n+1}$ and $1-q=1-\frac{K^2}{n+1}$, which gives
\eqref{eq:proj-error-T1-unscaled}.
\end{proof}

\begin{proposition}\label{prop:bargmann-bound}
Let assumptions of Proposition~\ref{prop:hermite-coeff} hold. For every $z\in\C^d$,
\begin{equation}\label{eq:bargmann-exact-T1}
(\mathcal B g^\star)(z)
=
\pi^{-d/4}\,2^{-d/2}\,
\int_{\R^d} w(x)\,
\exp\!\Big(
-\frac{|x|^2}{4}
+\frac{1}{\sqrt2}\,x\cdot z
\Big)\,dx .
\end{equation}
Moreover for every $r>0$,
\begin{equation}\label{eq:bargmann-polydisk-T1}
\sup_{\max_j |z_j|\le r}\,|(\mathcal B g^\star)(z)|
\ \le\
\pi^{-d/4}\,2^{-d/2}\,M_0\,
\exp\!\Big(\frac{R}{\sqrt2}\,\sqrt d\,r\Big).
\end{equation}
\end{proposition}
\begin{proof}
By definition and Fubini,
\[
(\mathcal B g^\star)(z)
=\pi^{-d/4}e^{-\frac12 z\cdot z}
\int_{\R^d} w(x)\,(2\pi)^{-d/2}\!
\int_{\R^d}
\exp\!\Big(-\tfrac12|y|^2+\sqrt2\,z\cdot y-\tfrac12|x-y|^2\Big)\,dy\,dx.
\]
Expand $|x-y|^2=|x|^2-2x\cdot y+|y|^2$ to obtain
\[
-\tfrac12|y|^2-\tfrac12|x-y|^2+\sqrt2\,z\cdot y
=
-|y|^2+(x+\sqrt2\,z)\cdot y-\tfrac12|x|^2.
\]
Complete the square:
\[
-|y|^2+(x+\sqrt2 z)\cdot y
=
-\Big|y-\frac{x+\sqrt2 z}{2}\Big|^2
+\frac{|x+\sqrt2 z|^2}{4}.
\]
Therefore,
\[
\int_{\R^d}\exp\!\Big(-|y|^2+(x+\sqrt2 z)\cdot y\Big)\,dy
=\pi^{d/2}\exp\!\Big(\frac{|x+\sqrt2 z|^2}{4}\Big).
\]
Hence
\begin{align*}
&(2\pi)^{-d/2}\int_{\R^d}
\exp\!\Big(-\tfrac12|y|^2+\sqrt2\,z\cdot y-\tfrac12|x-y|^2\Big)\,dy
\\
&\qquad
=(2\pi)^{-d/2}\,\pi^{d/2}\,
\exp\!\Big(-\frac{|x|^2}{2}+\frac{|x+\sqrt2 z|^2}{4}\Big)
=2^{-d/2}\exp\!\Big(-\frac{|x|^2}{4}+\frac{1}{\sqrt2}x\cdot z+\frac12 z\cdot z\Big),
\end{align*}
where we used $|x+\sqrt2 z|^2=|x|^2+2\sqrt2\,x\cdot z+2\,z\cdot z$.
Multiplying by the prefactor $e^{-\frac12 z\cdot z}$ from the Bargmann kernel cancels
the quadratic term in $z\cdot z$, yielding \eqref{eq:bargmann-exact-T1}.
From \eqref{eq:bargmann-exact-T1} and $e^{-|x|^2/4}\le 1$,
\[
|(\mathcal B g^\star)(z)|
\le
\pi^{-d/4}\,2^{-d/2}\int_{\R^d} |w(x)|\,
\Big|\exp\!\Big(\frac{1}{\sqrt2}x\cdot z\Big)\Big|\,dx
=
\pi^{-d/4}\,2^{-d/2}\int_{\R^d} |w(x)|\,
\exp\!\Big(\frac{1}{\sqrt2}x\cdot \Re z\Big)\,dx.
\]
If $\mathrm{supp}(w)\subseteq\{|x|\le R\}$, then $x\cdot \Re z\le |x|\,|\Re z|\le R\,|\Re z|$,
and thus
\[
|(\mathcal B g^\star)(z)|
\le
\pi^{-d/4}\,2^{-d/2}\,M_0\,
\exp\!\Big(\frac{R}{\sqrt2}|\Re z|\Big).
\]
On the polydisk $\max_j|z_j|\le r$ we have $|\Re z|\le |z|\le \sqrt d\,r$, hence
\[
\sup_{\max_j |z_j|\le r}|(\mathcal B g^\star)(z)|
\le
\pi^{-d/4}\,2^{-d/2}\,M_0\,
\exp\!\Big(\frac{R}{\sqrt2}\sqrt d\,r\Big),
\]
which is \eqref{eq:bargmann-polydisk-T1}. 
\end{proof}

\begin{lemma}
\label{bracketing number}
Let $\mathcal F_B = \{f_c(x) := \sum_{|\alpha| \le n} c_\alpha \psi_\alpha(x), \sum_{|\alpha| \le n} |c_{\alpha}|^2 \le B\}$. 
Then there exists some absolute constant $C > 0$ such that
$$
\log \mathcal N_{[]}(\mathcal F_B, \| \cdot \|_\infty, \varepsilon) \le p \log \bigg (1 + \frac{BC \sqrt{p}}{\varepsilon} \bigg )\, ,
$$
where $p = \binom{n+d}{d}$.
\end{lemma}
\begin{proof}
    Denote the set of coefficients by $\mathcal C:= \{c = (c_\alpha, |\alpha| \le n) \in \R^p \text{ with } \sum_{|\alpha| \le n} |c_{\alpha}|^2 \le B\}$. Note that
    $$
    \mathcal N(\mathcal C, \|\cdot\|_2, \varepsilon) \le \bigg (1 + \frac{2B}{\varepsilon} \bigg)^p,
    $$
    Let $f_{c}, f_{c^\prime} \in \mathcal F_B$. Then
    $$
    \| f_c - f_{c^\prime} \|_\infty \le \sqrt{p} \max_{|\alpha| \le n} \|\psi_\alpha\|_\infty \|c - c^\prime\|_2 \le C^d \sqrt{p} \|c - c^\prime\|_2,
    $$ 
    where $C$ is some absolute constant. 
    Hence,
    $$
    \mathcal N_{[]}(\mathcal F_B, \| \cdot \|_\infty, \varepsilon) \le \bigg (1 + \frac{4BC \sqrt{p}}{\varepsilon} \bigg)^p
    $$
\end{proof}

\begin{proposition}\label{prop:raising-direct}
Let $H_n$ be the \emph{physicists' Hermite polynomials} defined by the generating function
\[
\sum_{n=0}^\infty H_n(x)\,\frac{t^n}{n!}=e^{2xt-t^2},\qquad x,t\in\R.
\]
Define the (normalized) Hermite functions
\begin{equation}\label{eq:psi-phys}
\psi_n(x):=\frac{1}{(2^n n!\sqrt\pi)^{1/2}}\,H_n(x)\,e^{-x^2/2},
\qquad x\in\R,\ n\in\N_0,
\end{equation}
and the creation operator
\[
a^\ast:=\frac{1}{\sqrt2}\Big(x-\frac{d}{dx}\Big).
\]
Then for every $n\ge 1$,
\begin{equation}\label{eq:raising-identity}
a^\ast \psi_{n-1}=\sqrt n\,\psi_n,
\qquad\text{equivalently}\qquad
\psi_n=\frac{1}{\sqrt n}\,a^\ast\psi_{n-1}.
\end{equation}
\end{proposition}

\begin{proof}
From the generating function, differentiate with respect to $x$:
\[
\sum_{n=0}^\infty H_n'(x)\,\frac{t^n}{n!}
=\partial_x\big(e^{2xt-t^2}\big)=2t\,e^{2xt-t^2}
=2t\sum_{n=0}^\infty H_n(x)\,\frac{t^n}{n!}
=\sum_{n=1}^\infty 2n\,H_{n-1}(x)\,\frac{t^n}{n!}.
\]
Comparing coefficients of $t^n$ gives
\begin{equation}\label{eq:H-derivative}
H_n'(x)=2n\,H_{n-1}(x),\qquad n\ge 1.
\end{equation}
Next, differentiate the generating function with respect to $t$:
\[
\sum_{n=0}^\infty H_{n+1}(x)\,\frac{t^n}{n!}
=\partial_t\big(e^{2xt-t^2}\big)=(2x-2t)e^{2xt-t^2}
=2x\sum_{n=0}^\infty H_n(x)\,\frac{t^n}{n!}
-2\sum_{n=0}^\infty H_n(x)\,\frac{t^{n+1}}{n!}.
\]
Rewrite the last term as $\sum_{n=0}^\infty 2n\,H_{n-1}(x)\,\frac{t^n}{n!}$ and compare coefficients:
\begin{equation}\label{eq:H-recurrence}
H_{n+1}(x)=2x\,H_n(x)-2n\,H_{n-1}(x),\qquad n\ge 1.
\end{equation}
Using the product rule,
\[
\frac{d}{dx}\big(H_{n-1}(x)e^{-x^2/2}\big)
=H_{n-1}'(x)e^{-x^2/2}-xH_{n-1}(x)e^{-x^2/2}.
\]
Hence
\begin{align}
\Big(x-\frac{d}{dx}\Big)\big(H_{n-1}e^{-x^2/2}\big)
&=xH_{n-1}e^{-x^2/2}-\Big(H_{n-1}'e^{-x^2/2}-xH_{n-1}e^{-x^2/2}\Big)\notag\\
&=\big(2xH_{n-1}-H_{n-1}'\big)e^{-x^2/2}. \label{eq:key}
\end{align}
Now use \eqref{eq:H-derivative} with $n-1$:
\[
H_{n-1}'(x)=2(n-1)H_{n-2}(x).
\]
Insert this into \eqref{eq:key} and apply \eqref{eq:H-recurrence} with index $n-1$:
\[
2xH_{n-1}-H_{n-1}'
=2xH_{n-1}-2(n-1)H_{n-2}
=H_n.
\]
Therefore we have shown 
\begin{equation}\label{eq:raising-core}
\Big(x-\frac{d}{dx}\Big)\big(H_{n-1}(x)e^{-x^2/2}\big)=H_n(x)e^{-x^2/2}.
\end{equation}
Let
\[
N_n:=(2^n n!\sqrt\pi)^{-1/2},
\quad\text{so that}\quad
\psi_n=N_n\,H_n\,e^{-x^2/2}.
\]
Using \eqref{eq:raising-core},
\[
a^\ast\psi_{n-1}
=\frac{1}{\sqrt2}\Big(x-\frac{d}{dx}\Big)\big(N_{n-1}H_{n-1}e^{-x^2/2}\big)
=\frac{N_{n-1}}{\sqrt2}\,H_n\,e^{-x^2/2}.
\]
It remains to compare $\frac{N_{n-1}}{\sqrt2}$ with $\sqrt n\,N_n$:
\[
\frac{N_{n-1}}{\sqrt2}
=\frac{1}{\sqrt2}\,(2^{n-1}(n-1)!\sqrt\pi)^{-1/2}
=(2^n (n-1)!\sqrt\pi)^{-1/2},
\]
and
\[
\sqrt n\,N_n
=\sqrt n\,(2^n n!\sqrt\pi)^{-1/2}
=(2^n (n-1)!\sqrt\pi)^{-1/2}.
\]
Thus $\frac{N_{n-1}}{\sqrt2}=\sqrt n\,N_n$, and consequently
\[
a^\ast\psi_{n-1}=\sqrt n\,N_n\,H_n\,e^{-x^2/2}=\sqrt n\,\psi_n,
\]
which is \eqref{eq:raising-identity}.
\end{proof}

\section{Further Details of Numerical Experiments}
\label{sec:numerical_details}

In this section, we elaborate on details of numerical experiments presented in Section \ref{sec:numerical}.
The section is organized as follows. In Appendix \ref{sec:general}, we discuss general implementation details. Appendix \ref{sec:metrics} presents information about the metrics used and describes the algorithm. Appendix \ref{sec:toy} provides additional information about the Swiss-Roll to S-Curve two dimensional problem experiment. Appendix \ref{sec:25gauss} provides additional details about the 25 Gaussian mixture Extrapolation experiment. Appendix \ref{sec:Single_Cell} provides additional details about the biological data experiment. Finally, Appendix \ref{sec:hyperparams} presents all final hyperparameter values used for the experiments.

\subsection{General Implementation Details} \label{sec:general}

In our experiments, we paid close attention to hyperparameter selection. Due to the algorithm's specifics, careful hyperparameter tuning is essential for stable training and generation. This was accomplished by using Optuna \citep{Akiba2019}. We used 50 trial optimization for Optuna over the sliced $ \mathbb{W}_1 $ metric in the 25 Gaussian experiment and 100 trial optimization for Optuna over the sliced $ \mathbb{W}_1 $ metric for Single Cell data.

The calculations were fully performed on the Nvidia T4 GPU. For the SinkhornBridge algorithm, we used the hyperparameter values provided in the official repository, as we were unable to conduct extensive testing of the algorithm and select its hyperparameters ourselves.

\subsection{Metrics Used in the Experiments} \label{sec:metrics}

Sliced Wasserstein Distance was used as main metric since it provides quantitative, theoretically-grounded measures of distribution similarity and has been widely adopted in research papers and practical applications.

\[
\text{Sliced} ~ \mathbb{W}_1(X,Y) = \frac{1}{N}\sum_{n=1}^N \mathbb{W}_1(\{\langle x_i, \theta_n \rangle\}_{i=1}^M, \{\langle y_i, \theta_n \rangle\}_{i=1}^M),
\]
where $\{\theta_n\}_{n=1}^N $ are random projections uniformly distributed on $\mathbb{S}^{d-1}$.

For the Sliced Wasserstein Distance, 100 random projections were used.

We also provide a description of the training algorithm for the proposed algorithm training \ref{code:algo_train} and sampling \ref{code:algo_sample}. The pseudocode provided allows a reader to visually evaluate how our approach differs from SinkhornBridge.

\subsection{Details of Evaluation on the Swiss-Roll to S-Curve Experiment} \label{sec:toy} 

\begin{figure*}
    \centering
    \includegraphics[width=0.4\linewidth]{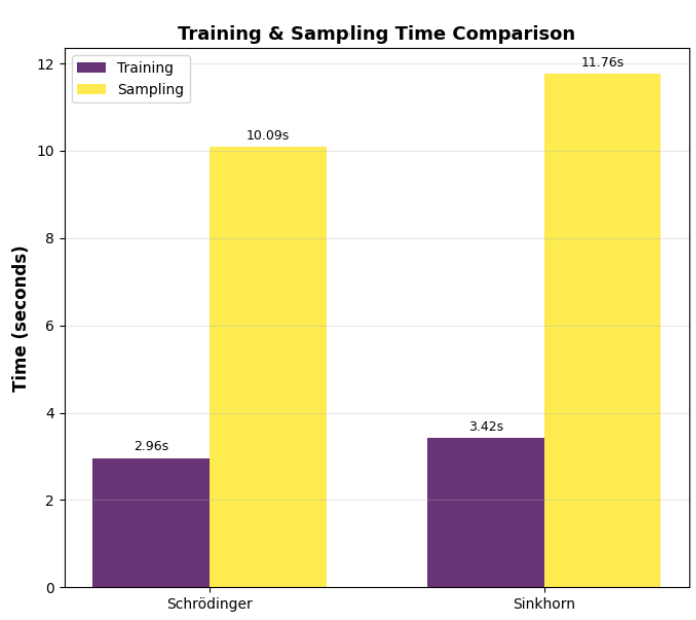}
    \caption{Comparison of training and sampling times for ERM-Bridge and SinkhornBridge on the Swiss-Roll to S-Curve translation task. We used 2000 training points for both algorithms.}
    \label{fig:SwissRoll_time}
\end{figure*}

In this experiment, we validated the algorithm's performance and sampling results at various intermediate time points on a common two dimensional problem in generative modelling. Optuna did not perform hyperparameter optimization for this problem, as it is not complex and both algorithms solve it well for any reasonable choice of training and sampling parameters.

We also compared the running time of our algorithm and SinkhornBridge on this problem, demonstrating significant sampling speedup.

\subsection{Details of Evaluation on the Gaussian Mixture} \label{sec:25gauss} 

%\begin{figure*}
%    \centering
%    \includegraphics[width=0.95\linewidth]{plots/samples_grid.png}
%    \caption{Transport samples from truncated normal on $ [-1, 1], [-2,2], [-3, 3], [-5,5] $ for ERM-Bridge in the top row and for SinkhornBridge in the bottom row. It is clear that the quality of SinkhornBridge degrades faster as the truncation boundaries are narrowed.}
%    \label{fig:25gauss_samples}
%\end{figure*}

In this experiment, we aimed to clearly demonstrate the behavior of our proposed algorithm and SinkhornBridge with explicit data bias in the sampling dataset relative to the train dataset. To this end, we created a synthetic example using a uniform grid of 25 Gaussians in two-dimensional space and a truncated Normal distribution of $ [-10, 10] $ in the train dataset and $ [-1, 1], [-3, 3], [-5, 5], [-10, 10] $ in test.

The problems arising from the discrete nature of SinkhornBridge become apparent upon visual inspection of the plots. The value of the experiment is that similar problems demonstrated in the synthetic example also arise with real data. For example, our algorithm demonstrated higher quality on Single Cell data \ref{sec:Single_Cell}, where both the initial and final distributions are only available as samples.

\subsection{Details of Evaluation on the Single Cell data} \label{sec:Single_Cell} 

In this experiment, we aimed to demonstrate the performance of our algorithm on the data-to-data translation task. The continuous log potential not only yielded the best metric value but was also easier to optimize, requiring less sampling time.

The results for the LightSB model were taken from the paper \citep{KorotinGushchinBurnaev2024}.

\subsection{Final Hyperparameter Values} \label{sec:hyperparams}

Below are the final hyperparameter values for all experiments.

\noindent \textbf{Swiss-Roll to S-Curve Experiment:} 
\begin{itemize}
\item Parameters for experiment: $ \text{batch size} = 1000, ~ ~ \text{lr} = 2e-3, ~ \text{epochs} = 1500, ~ \sigma\text{\_end} = 0.5, ~ \text{loss\_scale} = 1.0, ~ \text{hidden\_dim} = 64. $
\end{itemize}

\noindent \textbf{Gaussian Mixture Experiment:} 
\begin{itemize}
\item Parameters for experiment: $ \text{batch size} = 64, ~ ~ \text{lr} = 5e-4, ~ \text{epochs} = 140, ~ \sigma\text{\_end} = 0.9, ~ \text{loss\_scale} = 0.11 ~ \text{hidden\_dim} = 128. $
\end{itemize}

\noindent \textbf{Single Cell Experiment:} 
\begin{itemize}
\item Parameters for experiment: $ \text{batch size} = 2048, ~ ~ \text{lr} = 1e-4, ~ \text{epochs} = 141, ~ \sigma\text{\_end} = 0.4216, ~ \text{loss\_scale} = 196.5431, ~ \text{hidden\_dim} = 2048. $
\end{itemize}

\begin{algorithm} 
\caption{Training of ERM-Bridge}
\begin{algorithmic}[1]
\STATE  \textbf{Input:} Datasets $\mathcal{X}, \mathcal{Y}$, kernel bandwidth $\sigma$, learning rate $\eta$, batch size $B$, the number of iterations $ N_{steps}$
\STATE  \textbf{Initialize:} Neural network parameters $\theta$ for $\phi_\theta$
\STATE 
\FOR{$i$ from $1$ to $ N_{steps}$}
    \STATE Sample batches $X_b \sim \mathcal{X}$ and $Y_b \sim \mathcal{Y}$
    \STATE
    %\STATE // \textit{1. Forward Potential Evaluation:}
    \FOR{$y \in Y_b$}
        \STATE $V(y) \leftarrow \log \phi_\theta(y)$
    \ENDFOR
    \STATE
    %\STATE // \textit{2. Log-Domain Sinkhorn Iteration:}
    %\STATE // \textit{Compute backward marginal constraint (Log-Dual D):}
    \FOR{$x \in X_b$}
        \STATE $\text{term}(y) \leftarrow -\frac{\|x-y\|^2}{2\sigma^2} - V(y)$ \quad for all $y \in Y_b$
        \STATE $\log D(x) \leftarrow \log \sum_{y \in Y_b} \exp(\text{term}(y))$
    \ENDFOR
    \STATE
    %\STATE // \textit{Compute forward projection (Log-Primal C):}
    \FOR{$y \in Y_b$}
        \STATE $\text{term}(x) \leftarrow -\frac{\|y-x\|^2}{2\sigma^2} - \log D(x)$ \quad for all $x \in X_b$
        \STATE $\log (C\phi)(y) \leftarrow \log \sum_{x \in X_b} \exp(\text{term}(x))$
    \ENDFOR
    \STATE
    %\STATE // \textit{3. Compute Centered Loss:}
    \STATE $\Delta(y) \leftarrow V(y) - \log (C\phi)(y)$
    \STATE $\mathcal{L}(\theta) \leftarrow \frac{1}{|Y_b|} \sum_{y \in Y_b} \left( \Delta(y) - \text{mean}(\Delta) \right)^2$
    \STATE
    %\STATE // \textit{4. Update:}
    \STATE $\theta \leftarrow \theta - \eta \nabla_\theta \mathcal{L}(\theta)$
\ENDFOR
\STATE \textbf{Output:} Learned potential parameters $\theta^*$
\end{algorithmic} \label{code:algo_train}
\end{algorithm}

\begin{algorithm}
\caption{Sampling via Learned Continuous Drift}
\begin{algorithmic}[1]
\STATE \textbf{Input:} Trained model $\phi_{\theta^*}$, Initial sample $x_0$, Reference target samples $\mathcal{Y}_{ref}$, Total time $T$, Steps $K$
\STATE \textbf{Initialize:} $x \leftarrow x_0$, $t \leftarrow 0$, $\Delta t \leftarrow T/K$
\STATE
\FOR{$k = 0$ to $K-1$}
    \STATE Set noise level $\sigma_t$ (e.g., via cosine schedule)
    \STATE $\nu_t \leftarrow \sigma_t^2 (T - t)$
    \STATE
    %\STATE // \textit{1. Estimate Drift via Autodiff on Log-Potential:}
    \FOR{$j = 1 \dots |\mathcal{Y}_{ref}|$}
        \STATE $L_j \leftarrow -\frac{\|x - y_j\|^2}{2\nu_t} - \log \phi_{\theta^*}(y_j)$
    \ENDFOR
    \STATE $h(x) \leftarrow \log \sum_j \exp(L_j)$
    \STATE $g \leftarrow \nabla_x h(x)$
    \STATE $u(x,t) \leftarrow \sigma_t^2 \cdot g$
    \STATE
    %\STATE // \textit{2. Euler-Maruyama Step:}
    \STATE Sample noise $\xi \sim \mathcal{N}(0, I)$
    \STATE $x \leftarrow x + u(x,t)\Delta t + \sigma_t \sqrt{\Delta t} \cdot \xi$
    \STATE $t \leftarrow t + \Delta t$
\ENDFOR
\STATE \textbf{Return:} Transported sample $x_T$
\end{algorithmic} \label{code:algo_sample}
\end{algorithm}

\end{document}